\RequirePackage{fix-cm}
\documentclass[smallextended]{svjour3}       
\smartqed  
\usepackage[colorlinks,linkcolor=red,anchorcolor=green,citecolor=blue]{hyperref}
\usepackage{graphicx}
\usepackage{amsfonts,amssymb}
\usepackage{amsmath}
\usepackage{enumerate}
\usepackage[ruled,vlined]{algorithm2e}
\usepackage[labelformat=simple]{subcaption}

\usepackage{verbatim}

\usepackage{xcolor}
\usepackage{lineno}

\DeclareMathOperator*{\argmin}{arg\,min}
\begin{document}

\title{Decomposition of Longitudinal Deformations via Beltrami Descriptors \thanks{HKRGC  GRF  (Project  ID:2130549, Reference ID: 14306917), CUHK Direct Grant (Project ID: 4053292)}
}

\author{Ho LAW  \and
        Chun Yin SIU \and
        Lok Ming LUI
}


\institute{Ho LAW \at
            The Chinese University of Hong Kong \\
              \email{hlawab@connect.ust.hk}
           \and
           Chun Yin SIU \at
           Cornell University \\
           \email{cs2323@cornell.edu}
           \and
           Lok Ming LUI \at
           The Chinese University of Hong Kong \\
           \email{lmlui@math.cuhk.edu.hk}
}

\date{Received: date / Accepted: date}

\maketitle

\begin{abstract}
We present a mathematical model to decompose a longitudinal deformation into normal and abnormal components. The goal is to detect and extract subtle abnormal deformation from periodic motions in a video sequence. It has important applications in medical image analysis. To achieve this goal, we consider a representation of the longitudinal deformation, called the {\it Beltrami descriptor}, based on quasiconformal theories. The Beltrami descriptor is a complex-valued matrix. Each longitudinal deformation is associated to a Beltrami descriptor and vice versa. To decompose the longitudinal deformation, we propose to carry out the low rank and sparse decomposition of the Beltrami descriptor. The low rank component corresponds to the periodic motions, whereas the sparse part corresponds to the abnormal motions of a longitudinal deformation. Experiments have been carried out on both synthetic and real video sequences. Results demonstrate the efficacy of our proposed model to decompose a longitudinal deformation into regular and irregular components.
\end{abstract}

\section{Introduction}\label{sec: introduction}
Deformation analysis plays a significant role in medical image analysis\cite{zitova_2003,sotiras_2013}. Given a longitudinal medical image sequence, the spatio-temporal analysis can be carried out through studying deformations between images, which is useful to understand the pathology for disease analysis. In particular, this paper aims to filter out normal longitudinal deformations from the abnormal which could be caused by certain diseases in a medical video. By normal deformations, we refer to a series of deformations that follows a periodic motion. Suppose there is a medical footage that records a periodic motion, for instance, a beating heart under a normal cardiac motion. Then the normal deformation is this periodic motion. On the other hand, some patients may occasionally suffer from some diseases like Premature Atrial Contraction and Premature Ventricular Contraction that would perturb this kind of normal cardiac deformation\cite{heaton_2020}. In this case, we call this kind of perturbations the abnormal deformation. Our goal is to extract the abnormal deformation (perturbations) from the normal (periodic) deformation, which are originally combined together. In order to analyze the deformities efficiently and accurately, the capability to decompose a longitudinal deformation into regular and irregular motions is necessary. For instance, during normal cycles of contraction and expansion of a lung when breathing, some parts of the lung may tremble unnaturally\cite{gibson_1984}. Combined with the normal motion, doctors might have difficulty to discern the abnormal motion. It thus calls for the need of a mathematical model to the decomposition of longitudinal deformation into normal and abnormal components.

To achieve this goal, an effective representation of the longitudinal deformation is necessary. An intuitive representation is based on the deformation vector fields obtained via image registration techniques. As vector fields cannot effectively capture the geometric information of deformations, the decomposition based on vector fields is usually unable to extract meaningful regular and irregular components and evidence will be provided in later section. The difficulty is that the extracted components are not in a bijective correspondence with flips or overlaps, which are unnatural and unrealistic for deformations of anatomical structures. In this work, we propose to consider a representation of the longitudinal deformation, called the {\it Beltrami descriptor}, in quasiconformal theories\cite{gardiner_2000,lui_2012,chan2018topology}. The Beltrami descriptor is a complex-valued matrix, which captures the geometric information of the longitudinal deformation. Hence, the geometric distortion and bijectivity of the deformation can be easily controlled\cite{lam_2014}. More importantly, it is an effective representation since each longitudinal deformation is associated to a unique Beltrami descriptor and vice versa. The associated deformation is also stable under the perturbation of the descriptor. As such, the manipulation of the longitudinal deformations through the Beltrami descriptors is not sensitive to the error of the descriptors, which is crucial for the decomposition. To decompose the longitudinal deformation, we propose to extract the low rank part and sparse part of the Beltrami descriptor. The periodic motion of the deformation is characterized by the low rank component of the descriptor. On the other hand, abnormal deformation is characterized by the sparse part of the Beltrami descriptor. This low rank and sparse pursuit problem can be relaxed to a complex-valued Robust Principal Component Analysis (RPCA) problem\cite{cai_2008,wright_2009}, which can be solved by alternating minimization method with multipliers (ADMM)\cite{yuan_2009}. We test our proposed model on both synthetic and real video sequences. Experimental results illustrate the efficacy of our proposed method for the decomposition of longitudinal deformations.

In short, our contributions of this paper are three-folded.
\begin{enumerate}
    \item First, we propose to consider a special representation of longitudinal deformations, called the Beltrami descriptor, to decompose the deformation. The Beltrami descriptor captures the geometric information of the deformation, and hence manipulating the descriptor allows us to process and analyze the deformation according to its geometry.
    \item Secondly, we consider the low rank and sparse decomposition of the Beltrami descriptor to the decomposition of longitudinal deformation into regular and irregular components. To the best of our knowledge, it is the first work to the decomposition of longitudinal deformation via low rank and sparse pursuit. 
    \item Thirdly, in practical applications, it is often desirable to extract bijective irregular longitudinal component, which detect and capture the abnormal deformation from normal periodic motion. In this work, we theoretically show that the extracted irregular component is bijective under a suitable choice of parameters.    
\end{enumerate}

The paper is organized as follow: in section \ref{sec: previous work}, we will briefly review some previous works related to this paper. In section \ref{sec: math bg}, some necessary mathematical tools will be described. The Beltrami descriptor and our proposed decomposition algorithm will be explained in details in section \ref{sec: decomposition of deformations}. Last but not least, experimental results will be shown in section \ref{sec: experiment}, and we cap off with a conclusion and discussion of future works in section \ref{sec: conclusion}. \par

\section{Previous Work}\label{sec: previous work}

Shape analysis of structures from images plays a fundamental role in various fields, such as computer visions and medical image analysis. One commonly used approach is based on the analysis of the deformation fields between corresponding images. Deformation fields between images are often obtained through the image registration process. Registration aims to establish a meaningful one-to-one dense correspondence between images. Over years, various registration methods have been proposed, which can be categorized into feature-based \cite{yasein_2008,islam_2013,lee2016landmark,lui2014teichmuller}, intensity-based \cite{vercauteren_2009,trouve_2005}, and combined-feature-intensity-based methods \cite{yao_2001,lam_2014,lui2014geometric}. Amongst these methods, quasiconformal-based registration models have been widely used \cite{lui2014geometric,lee2016landmark,qiu2020inconsistent,LuiHP,lui2010optimized,lam_2014,wang2007brain,lui2012optimization,choi2015flash,lui2014teichmuller},, with which our model in this paper is built upon. For instance, in \cite{lam_2014}, Lam et al. proposed an optimization model based on quasiconformal geometry to obtain landmark-based and intensity-based registration between images or surfaces.

Once the deformation fields are obtained, different shape analysis methods have been recently proposed. In \cite{LuiHP,LuiIPI,lui_2012,LuiTooth,LuiWaveletBC,LuiQCAD,choi2020shape,zeng2010shape,lui2013shape,chan2020quasi} , Lui et al.  proposed to detect shape variation based on the Beltrami coefficients of the deformation field as well as the curvature mismatching. The method has been applied for Alzhemier's disease analysis \cite{LuiQCAD} and tooth morphometry \cite{LuiTooth}. A quasiconformal metric for deformation classification is also introduced to classify the left ventricle deformations of myopathic and control subjects \cite{Taimuri}. The wavelet support vector machine (WSVM) has been proposed to study the deformation field \cite{PengWavelet}. Algorithms to analyze deformation field with different geometric scales and directions have also been recently developed. The basic idea is to decompose the vector field representing the deformation into various meaningful components. For instance, Tong et al. \cite{TongMultiVector} proposed a variational model to decompose a vector field into the divergence-free part, the curl-free part, and the harmonic part using the idea of Helmholtz-Hodge decomposition. Recently, the morphlet transform has been proposed to obtain a multi-scale representation for diffeomorphisms \cite{Morphlet}. Wavelet tranform on the Beltrami coefficient of the deformation field has also been proposed to decompose a deformation into multiple components with various geometric scales \cite{LuiWaveletBC}. However, to the best of our knowledge, an effective method to analyze time-dependent longitudinal deformation is still lacking.

In this work, our goal is to decompose a longitudinal deformation into normal and abnormal components. To do so, Robust Principal Component Analysis (RPCA) will be performed on the descriptor of the longitudinal deformation. RPCA has been widely studied in recent years and have been used for various applications. For example, Zhou et. al.\cite{zhou_2011} proposed ``GoDec'' that was adding one more noise term, so as to remove the noise captured by cameras. Also, Zhou et. al.\cite{zhou_2013} made an improvement by imposing one more constraint to ensure the moving objects are small and continuous pieces. Li et. al.\cite{li_2019} suggested another method, SSC-RPCA, that could work well when the background exhibits some minor motion, like flowing water of a lake or a river, or the moving object does not move fast enough, with more terms into the original RPCA model to force the model to group different regions of the moving object in a roughly segmented video. Oreifej et. al.\cite{oreifej_2013} presented another term to model turbulence to capture moving object in a badly turbulence-corrupted video. Sobral et. al.\cite{sobral_2015} proposed a way to improve detection of moving object by imposing shape constraints. Javed et. al.\cite{javed_2015} put forward a superpixel-based matrix decomposition method with maximum norm regularizations and structured sparsity constraints to deal with the real-time challenge. The model designed by Ebdai et. al.\cite{ebadi_2018} estimates the support of the foreground regions with a superpixel generation step, and then spatial coherence can be imposed. Cao et. al.\cite{cao_2016} presented a novel method of RPCA, using tensor decomposition, as well as 3D total variation to enforce spatio-temporal continuity of the moving objects.

To compute the RPCA effectively, various numerical methods have been proppsed. For example, Lin et. al.\cite{lin_2009} compared two methods: accelerated proximal gradient algorithm applied to the primal and gradient algorithm applied to the dual problem. Another well-know optimization method, which is going to be used in this paper, is the {\bf A}lternating {\bf D}irection {\bf M}ethod ({\bf ADM}) proposed by Yuan et. al.\cite{yuan_2009}, or similarly the Augmented Lagrange Multiplier Method proposed by Lin et. al.\cite{lin_2010}.

\section{Mathematical Background}\label{sec: math bg}
In this section, we will review some mathematical background related to this work.

\begin{figure}[t]
\centering
\includegraphics[height=1.35in]{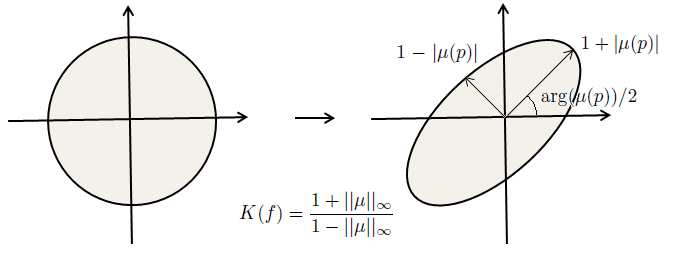}
\caption{Illustration of how the Beltrami coefficient determines the conformality distortion.}
\label{fig: illustration1}
\end{figure}

\subsection{Quasiconformal theories}
In the following, some basic ideas of quasiconformal geometry are discussed. For details, we refer readers to \cite{gardiner_2000,lehto_2011}.

A surface $S$ with a conformal structure is called a \emph{Riemann surface}. Conformal structure on a Riemann surface, $S$, is an equivalence class of metrics.

\begin{equation}\label{eq: def conformal structure}
    [h] = \{ e^{2u}h \ | \ u \in C^{\infty}(S) \},
\end{equation}
\noindent where $h$ is some Riemannian metric on $S$.
Given two Riemann surfaces $M$ and $N$, a map $f:M\to N$ is \emph{conformal} if it preserves the surface metric up to a multiplicative factor called the {\it conformal factor}. An immediate consequence is that every conformal map preserves angles. With the angle-preserving property, a conformal map effectively preserves the local geometry of the surface structure. 
A generalization of conformal maps is the \emph{quasiconformal} maps, which are orientation preserving homeomorphisms between Riemann surfaces with bounded conformality distortion, in the sense that their first order approximations take small circles to small ellipses of bounded eccentricity \cite{gardiner_2000}. Mathematically, $f \colon \mathbb{C} \to \mathbb{C}$ is quasiconformal provided that it satisfies the Beltrami equation:
\begin{equation}\label{beltramieqt}
\frac{\partial f}{\partial \overline{z}} = \mu(z) \frac{\partial f}{\partial z}.
\end{equation}
\noindent for some complex-valued function $\mu$ satisfying $||\mu||_{\infty}< 1$, where $\mu$ is called the \emph{Beltrami coefficient}, which is a measure of non-conformality. More precisely, it measures how far the map at each point is deviated from a conformal map. In particular, the map $f$ is conformal around a small neighborhood of $p$ when $\mu(p) = 0$. When $\mu(z) \equiv 0$, the Beltrami's equation becomes the Cauchy-Riemann equation and so the map is a conformal map. As such, a quasiconformal map can be regarded as a generalization of conformal map that allows bounded conformality distortions. Infinitesimally, around a point $p$, $f$ may be expressed with respect to its local parameter as follows:
\begin{equation}
\begin{split}
f(z) & = f(p) + f_{z}(p)z + f_{\overline{z}}(p)\overline{z} \\
& = f(p) + f_{z}(p)(z + \mu(p)\overline{z}).
\end{split}
\end{equation}

Obviously, $f$ is not conformal if and only if $\mu(p)\neq 0$. Locally, $f$ may be considered as a map composed of a translation to $f(p)$ together with a stretch map $S(z)=z + \mu(p)\overline{z}$, which is composed by a multiplication of $f_z(p),$ which is conformal. All the conformal distortion of $S(z)$ is caused by $\mu(p)$. $S(z)$ is the map that causes $f$ to map a small circle to a small ellipse. From $\mu(p)$, we can determine the angles of the directions of maximal magnification and shrinking and the amount of them as well. Specifically, the angle of maximal magnification is $\arg(\mu(p))/2$ with magnifying factor $1+|\mu(p)|$; the angle of maximal shrinking is the orthogonal angle $(\arg(\mu(p)) -\pi)/2$ with shrinking factor $1-|\mu(p)|$. Thus, the Beltrami coefficient $\mu$ gives us lots of information about the properties of the map (see Figure \ref{fig: illustration1}).

The maximal dilation of $f$ is given by:
\begin{equation}
K(f) = \frac{1+||\mu||_{\infty}}{1-||\mu||_{\infty}}.
\end{equation}

Given a Beltrami coefficient $\mu:\mathbb{C}\to \mathbb{C}$ with $\|\mu\|_\infty < 1$. There always exists a quasiconformal mapping from $\mathbb{C}$ onto itself which satisfies the Beltrami equation in the distributions sense \cite{gardiner_2000}. More precisely, we have the following theorem: \par

\bigskip

\begin{theorem}[Measurable Riemann Mapping Theorem] \label{thm: Beltrami}
Suppose $\mu: \mathbb{C} \to \mathbb{C}$ is Lebesgue measurable satisfying $\|\mu\|_\infty < 1$, then there exists a quasiconformal homeomorphism $\phi$ from $\mathbb{C}$ onto itself, which belongs to the Sobolev space $W^{1,2}(\mathbb{C})$ and satisfies the Beltrami Equation (\ref{beltramieqt}) in the distributions sense. Furthermore, by fixing 0, 1 and $\infty$, the associated quasiconformal homeomorphism $\phi$ is uniquely determined.
\end{theorem}

\bigskip

Theorem \ref{thm: Beltrami} suggests that under a suitable normalization, a homeomorphism from $\mathbb{C}$ or $\mathbb{D}$ onto itself can be uniquely determined by its associated Beltrami coefficient.

\subsection{Robust Principal Component Analysis (RPCA)}
The RPCA problem is stated as follow: Suppose we are given a matrix $M \in \mathbb{R}^{m \times n}$. Then we would like to solve the following minimisation problem:
\begin{align}\label{eq: initial rpca}
\min\limits_{\mathcal{N},\mathcal{A}} \mathrm{rank}(\mathcal{N}) + \alpha \|\mathcal{A}\|_{0}  \text{,  such that } M=\mathcal{N}+\mathcal{A};
\end{align}
where $\mathcal{N}$ and $\mathcal{A}$ are supposed to be a low-rank and a sparse matrix respectively, and $\alpha$ a parameter describing the trade off between the rank of the low-rank matrix and the $L_{0}$ norm of the sparse matrix. Since the above problem is NP-hard, we make appeal to a common relaxation, namely:
\begin{align}\label{eq: relaxed rpca}
\min\limits_{\mathcal{N},\mathcal{A}} \| \mathcal{N} \|_{\ast} + \alpha \|\mathcal{A}\|_{1}  \text{,  such that } M=\mathcal{N}+\mathcal{A}
\end{align}
Given that Equation (\ref{eq: relaxed rpca}) is a convex optimization problem, the ADM approach suggested by Yuan et. al.\cite{yuan_2009} is a suitable method. Namely, the Augmented Lagrangian function of Equation (\ref{eq: relaxed rpca}) is:
\begin{equation}\label{eq: aug lagrangian}
\mathcal{L}(\mathcal{N},\mathcal{A},Z;M) = \|\mathcal{N}\|_{\ast} + \alpha \|\mathcal{A}\|_{1} - \langle Z, \mathcal{N}+\mathcal{A}-M \rangle + \frac{\beta}{2}\|\mathcal{N}+\mathcal{A}-M\|^{2}_{2}
\end{equation}
where $Z$ is the multiplier of the linear constraint, $\beta$ is the penalty parameter. Here, we use $\langle \cdot, \cdot \rangle$ to denote the trace inner product. A simple iterative scheme is as follow:
\begin{equation}\label{eq: iter scheme}
\begin{cases}
\mathcal{N}^{k+1} = \argmin_{\mathcal{N} \in \mathbb{R}^{m \times n}} \mathcal{L}(\mathcal{N},\mathcal{A}^{k},Z^{k};M) \\
\mathcal{A}^{k+1} = \argmin_{\mathcal{A} \in \mathbb{R}^{m \times n}} \mathcal{L}(\mathcal{N}^{k+1},\mathcal{A},Z^{k};M) \\
Z^{k+1} = Z^{k} - \beta(\mathcal{N}^{k+1}+\mathcal{A}^{k+1}-M) \\
\end{cases}
\end{equation}
\cite{yuan_2009,ma_2009,cai_2008,tibshirani_1996} showed that there are closed formulas to update $\mathcal{N}^{k+1},\mathcal{A}^{k+1}$ and, obviously, $Z^{k+1}$ at each step. To solve for $\mathcal{A}^{k+1}$, we can use the explicit solution:
\begin{equation}\label{eq: update S}
   \mathcal{A}^{k+1} = \frac{1}{\beta}Z^{k} - \mathcal{N}^{k} + M - P_{\Omega^{\alpha / \beta}_{\infty}}\left( \frac{1}{\beta} Z^{k} - \mathcal{N}^{k} + M \right)
\end{equation}
where $P_{\Omega^{\alpha / \beta}_{\infty}}$ denoted the Euclidean projection onto $\Omega^{\alpha / \beta}_{\infty} := \{ X \in \mathbb{R}^{n \times n} \ | \ -\alpha / \beta \leq X_{ij} \leq \alpha / \beta \}$. For the subproblem $\mathcal{N}^{k+1}$, the explicit solution is
\begin{equation}\label{eq: update L}
   \mathcal{N}^{k+1} = U^{k+1}\mathrm{diag} \left( \max\{ \sigma_{i}^{k+1} - \frac{1}{\beta},0 \} \right) (V^{k+1})^{T}
\end{equation}
where $U^{k+1},V^{k+1},\sigma_{i}^{k+1}$ are obtained by SVD that is:
\begin{equation}\label{eq: svd update}
   M - \mathcal{A}^{k+1} + \frac{1}{\beta} Z^{k} = U^{k+1}\Sigma^{k+1}(V^{k+1})^{T} \quad \text{with} \quad \Sigma^{k+1} = \mathrm{diag} \left( \{ \sigma_{i}^{k+1} \}_{i=1}^{r} \right) 
\end{equation}


\section{Decomposition of Longitudinal Deformations}\label{sec: decomposition of deformations}

In this section, we explain our proposed main algorithm for the decomposition of longitudinal deformations. The goal is to separate abnormal deformations from normal deformations. To achieve this goal, it is necessary to have an effective representation of longitudinal deformations. The longitudinal deformation has to be easily restored from the corresponding representation. In addition, an effective algorithm to decompose the representation is also required. \par

\subsection{Representation of longitudinal deformations.}
In this work, we consider to represent the longitudinal deformations based on quasiconformal theories. An effective representation of longitudinal deformations should satisfy the following criteria. 

\begin{enumerate}
    \item First, the representation should capture the geometric information about the deformations. More precisely, it should describe the local geometric distortions created by the deformation mappings, so that the decomposed result should contain information about geometric distortion.
    \item The corresponding longitudinal deformations can be restored from the representation by Linear Beltrami Solver(LBS) which will be detailed in later section, so that the deformation fields can be obtained after the decomposition of the representation is carried out.
    \item The bijectivity of the corresponding deformations can be controlled during the manipulation of the representation with theoretical support. In other words, the corresponding deformations will not be severely corrupted during the decomposition process of the representation.
\end{enumerate}

\noindent To achieve these objectives, we will consider a longitudinal deformation matrix based on the Beltrami coefficients.  Suppose $\{I_i\}_{i=1}^t$ are the video frames, each of size $m\times n$, capturing the longitudinal data. Let $I_{ref}$ be a reference image. For each frame $I_j$, we compute the image registration $f_j: \Omega \to \Omega$ from $I_{ref}$ to $I_j$. Here, $\Omega$ refers to the rectangular image domain. The image registration can be computed using existing registration algorithms. In this work, the quasiconformal image registration method is applied.

Note that the image domain $\Omega$ is discretized into uniformly distributed pixels. As such, we can consider that $\Omega$ is discretized by regular triangulation $\{V,E,F\}$, where $V$ is the collection of vertices given by pixels. $E$ and $F$ are the collections of edges and faces respectively. With these notations, we assume $f_i:= ({\bf u}_i, {\bf v}_i)$, where ${\bf u}_i:V\to \mathbb{R}$ and ${\bf v}_i:V\to \mathbb{R}$ are the coordinate functions defined on every vertices. $f_i$ is regarded as piecewise linear on each face. The quasiconformality or local geometric distortion of $f_i$ can then be measured by the Beltrami coefficient. 

For the piecewise linear map $f_i$, we compute its Beltrami Coefficient by the approximation of its partial derivatives on each face $T\in F$. The restriction of $f_i$ on each face $T$ can be written as
\begin{equation}\label{eq: f on T}
   f_i|_{T}(x,y) = 
   \begin{pmatrix}
    a_{T}x+b_{T}y+r_{T} \\
    c_{T}x+d_{T}y+q_{T} \\
   \end{pmatrix}
\end{equation}
Hence, $D_{x}f_i(T) = a_{T} + ic_{T}$ and $D_{y}f_i(T) = b_{T} + id_{T}$. Then the gradient $\nabla_{T}f_i$ can be obtained by solving:
\begin{equation}\label{eq: grad f on T}
   \begin{pmatrix}
    v_1 - v_0 \\
    v_2 - v_0 \\
   \end{pmatrix}
   \begin{pmatrix}
    a_T & b_T \\
    c_T & d_T \\
   \end{pmatrix} = 
   \begin{pmatrix}
    {\bf u}(v_1) - {\bf u}(v_0) & {\bf u}(v_2) - {\bf u}(v_0)\\
    {\bf v}(v_1) - {\bf v}(v_0) & {\bf v}(v_2) - {\bf v}(v_0) \\
   \end{pmatrix}
\end{equation}
where $v_0, v_1$ and $v_2$ are the three vertices of the face $T$. By solving the above linear system, $a_{T},b_{T},c_{T},d_{T}$ can be computed. And the Beltrami coefficient of $f_i$ on $T$ can be obtained by
\begin{equation}\label{eq: discrete bc}
   \mu_i(T) = \frac{(a_{T}-d_{T}) + i(c_{T}+b_{T})}{(a_{T}+d_{T}) + i(c_{T}-b_{T})} 
\end{equation}

We thus have the following definition of {longitudinal deformation descriptor} to represent the longitudinal deformations.

\begin{definition}[Longitudinal deformation descriptor]\label{def:BCdeformationmatrix}
With the notations above, the longitudinal deformation descriptor $\mathcal{L}^{\mu}$ for $\{f_i\}_{i=1}^t$ is a $mn\times t$ complex-valued matrix given by
\begin{equation}
    \mathcal{L}^{\mu} = \begin{pmatrix}
\mid &\mid  &  &\mid \\ 
\mu_{1} &\mu_{2}  & \cdots & \mu_{t}\\ 
\mid &\mid  &  &\mid \\ 
\end{pmatrix}
\end{equation}
\end{definition}

$\mathcal{L}^{\mu}$ is formulated using Beltrami coefficients, which capture the local geometric distortions under the longitudinal deformations. As it will be explained in the next subsection, $\mathcal{L}^{\mu}$ has a one-one correspondence with the longitudinal deformations. In other words, given $\mathcal{L}^{\mu}$, the associated longitudinal deformations can be reconstructed. On the other hand, according to quasiconformal theories, the deformation $f_j$ is bijective (or folding-free) if $||\mu_j||_{\infty}<1$.

\subsection{Reconstruction of longitudinal deformations from descriptors.}
In the last subsection, we introduce the descriptor $\mathcal{L}^{\mu}$ to represent the longitudinal deformations. In order to utilize the descriptor, a reconstruction algorithm from the descriptor to the corresponding longitudinal deformations is necessary. 

Let's discuss how the longitudinal deformations can be reconstructed from $\mathcal{L}^{\mu}$. Consider $f_j|_{T}$ restricted to a triangle $T\in F$. Suppose the three vertices of $T$ is given by $v_0, v_1$ and $v_2$, whose coordinates are given by $v_k = (g_k, h_k)$ for $k=0,1$ or $2$. $v_0, v_1$ and $v_2$ are deformed by $f_j|_{T}$ to $w_0, w_1$ and $w_2$, whose coordinates are given by $w_k = (s_k,t_k)$ for $k=0,1,2$. Denote $\mu_j(T)= \rho_j + i \tau_j$. Let $\gamma_1(T) = \frac{(\rho_T-1)^2+\tau_T^2}{1-\rho_T^2-\tau_T^2}$, $\gamma_2(T) = \frac{-2\tau_T}{1-\rho_T^2-\tau_T^2}$ and $\gamma_1(T) = \frac{(1+\rho_T)^2+\tau_T^2}{1-\rho_T^2-\tau_T^2}$.

By comparing the real and imaginary parts, Equation (13) can be formulated as follows:
\begin{equation}
    \begin{split}
        a_T &= \alpha_T^0 s_0 + \alpha_T^1 s_1 + \alpha_T^2 s_2;\\
        b_T &= \beta_T^0 s_0 + \beta_T^1 s_1 + \beta_T^2 s_2;\\
        c_T &= \alpha_T^0 t_0 + \alpha_T^1 t_1 + \alpha_T^2 t_2;\\
        d_T &= \beta_T^0 t_0 + \beta_T^1 t_1 + \beta_T^2 t_2.
    \end{split}
\end{equation}
\noindent where 
\begin{equation}
    \begin{split}
        \alpha_T^0 &= (h_2-h_3)/\mathcal{A}_T; \alpha_T^1 = (h_2-h_0)/\mathcal{A}_T; \alpha_T^2 = (h_0-h_1)/\mathcal{A}_T;\\
        \beta_T^0 &= (g_2-g_3)/\mathcal{A}_T; \beta_T^1 = (g_2-g_0)/\mathcal{A}_T; \beta_T^2 = (g_0-g_1)/\mathcal{A}_T;
    \end{split}
\end{equation}
\noindent Here, $\mathcal{A}_T$ refers to the area of $T$. According to computational Quasiconformal Teichm\"{u}ller Theory \cite{gardiner_2000}, $a_T$, $b_T$, $c_T$ and $d_T$ also satisfy the following linear equations:
\begin{equation}
\begin{split}
\sum_{T\in N_i} \alpha_T^i[\gamma_1(T) a_T + \gamma_2(T) b_T]  + \beta_T^i[\gamma_2(T) a_T + \gamma_3(T) b_T] &= 0 ;\\
     \sum_{T\in N_i} \alpha_T^i[\gamma_1(T) c_T + \gamma_2(T) d_T]  + \beta_T^i[\gamma_2(T) c_T + \gamma_3(T) d_T] & = 0;    
\end{split}
\end{equation}
\noindent where $N_i$ denotes the set of faces attached to the vertex $v_i$.  Combining Equation (18) and (20), we obtain a linear system to solve for the coordinate functions ${\bf u}_j$ and ${\bf v}_j$ of $f_j$, subject to a given boundary condition. In practice, we usually set $f_j$ to be an identity map on the boundary as the boundary condition. Hence, we have $D_j f_j = D_j  ({\bf u}_j, {\bf v}_j) = ({\bf b}^1_j, {\bf b}^2_j)$, where $D_j$ is a $mn\times mn$ matrix $D_j$ and $({\bf b}^1_j, {\bf b}^2_j)$ is a $mn\times 2$ matrix given by the above non-singular linear system.

In summary, given $\mathcal{L^{\mu}}$, one can reconstruct the longitudinal deformations via solving a big linear system:
\begin{equation}\label{LBS}
    {\widetilde{\mathcal{D}}} {\bf f} = \begin{pmatrix}
D_1 &  &  & \\ 
 &D_{2}  &  & \\ 
&  & \ddots & \\
& & & D_t
\end{pmatrix}\begin{pmatrix}f_1\\
f_2\\
\vdots\\
f_t
\end{pmatrix} = \begin{pmatrix}
\mid & \mid\\
{\bf b}^1 & {\bf b}^2 \\
\mid & \mid
\end{pmatrix}:= {\bf b}
\end{equation}
\noindent where $\widetilde{\mathcal{D}}$ is a $mnt\times mnt$ block diagonal matrix, where $t$ is the number of frames, and hence the linear system can be solved in parallel, subject to the Dirichlet boundary condition that the map is an identity map on the boundary.


The above discussion gives rise to the following theorem about the relationship between the longitudinal deformation and its associated descriptors.

\bigskip

\begin{theorem}
 Let denote the longitudinal deformations by ${\bf f}$. To ${\bf f}$ is associated with a unique descriptor $\mathcal{L}^{\mu}$, given by Equation (13), that satisfies $||\mathcal{L}^{\mu}||_{\infty}<1$. Conversely, given a descriptor $\mathcal{L}^{\mu}$ of a  longitudinal deformation, the corresponding longitudinal deformation ${\bf f}$ can be exactly reconstructed and is unique. In other words, if a longitudinal deformation ${\bf g}$ has a descriptor given by $\mathcal{L}^{\mu}$, then ${\bf f}= {\bf g}$.
\end{theorem}

On the other hand, the bijectivity of the longitudinal deformation can be easily controlled by the norm of its descriptor. It can be explained by the following theorem.

\bigskip

\begin{theorem}\label{thm: bijective}
If $||\mathcal{L}^{\mu}||_{\infty}<1$, then its associated longitudinal deformation is bijective.
\end{theorem}
\begin{proof}
Note that $||\mathcal{L}^{\mu}||_{\infty} = \max_{i,j} \{ |(\mathcal{L}^{\mu})_{ij}|\}$, where $(\mathcal{L}^{\mu})_{ij}$ denotes the $i$-th row and $j$-th column entry of $\mathcal{L}^{\mu}$. Since $||\mathcal{L}^{\mu}||_{\infty}<1$, we have $||\mu_j||_{\infty} < 1$ for all $j=1,2,...,l$. For every triangular face $T$, the restriction map $f_j|_T$ on $T$ is a linear map. The Jacobian $J_T$ of $f_j|_T$ is given by
\[
J_T = |\frac{\partial}{\partial z} (f_j|_T)|^2 - |\frac{\partial}{\partial \bar{z}} (f_j|_T)|^2 = |\frac{\partial}{\partial z} (f_j|_T)|^2 (1-|\mu_j(T)|^2) >0
\]
\noindent since $|\mu_j(T)| = |\frac{\partial f_j|T}{\partial \bar{z}}|/ |\frac{\partial f_j|T}{\partial z}| < 1$ and $|\frac{\partial}{\partial z} (f_j|_T)|>0$ for a well-defined $\mu_j$. we conclude that $f_j|_T$ is orientation-preserving. Thus the piecewise linear deformation $f_j$ is locally injective on every one-ring neighborhood of a vertex. By Hadamard theorem, $f_j$ is globally bijective for all $j=1,2,...,l$. We conclude that the longitudinal deformation associated to $\mathcal{L}^{\mu}$ is bijective.
\end{proof}

\bigskip

In addition, it is important to understand how the difference in two descriptors related to the difference in their corresponding longitudinal deformations.

\bigskip

\begin{theorem}\label{thm: longitudinal deformation}
Let $\mathcal{L}^{\mu}_1$ and $\mathcal{L}^{\mu}_2$ be the desriptors of two longitudinal deformations ${\bf f}$ and ${\bf g}$ respectively. Suppose $||\mathcal{L}^{\mu}_1-\mathcal{L}^{\mu}_2||_F<\epsilon$, where $||\cdot||_F$ denotes the Frobenius norm. Then:
\begin{equation}
    \begin{split}
        ||{\bf f}-{\bf g}||_F & < C_1 \epsilon\\
        ||\mathcal{D}{\bf f}-\mathcal{D}{\bf g}||_F &< C_2 \epsilon
    \end{split}
\end{equation}
\noindent for some positive constants $C_1$ and $C_2$. Here, 
\[\mathcal{D}{\bf f} = \begin{pmatrix}
\mid & \mid & \mid & \mid & & \mid &\mid\\
\mathcal{D}_1 f_1 & \mathcal{D}_2 f_1 & \mathcal{D}_1 f_2 & \mathcal{D}_2 f_2 & \cdots & \mathcal{D}_1 f_t & \mathcal{D}_2 f_t\\
\mid & \mid & \mid & \mid & & \mid &\mid
\end{pmatrix}\in M_{|F|\times 2t} \]

\noindent where $\mathcal{D}_1\varphi = \begin{pmatrix}
\frac{\partial \varphi}{\partial z}(T_1)\\
\vdots\\
\frac{\partial \varphi}{\partial z}(T_{|F|})
\end{pmatrix}\in \mathbb{C}^{|F|}
$ and $\mathcal{D}_2 \varphi = \begin{pmatrix}
\frac{\partial \varphi}{\partial \bar{z}}(T_1)\\
\vdots\\
\frac{\partial \varphi}{\partial \bar{z}}(T_{|F|})
\end{pmatrix}\in \mathbb{C}^{|F|}
$, where $\varphi$ is a piecewise linear map on $\Omega$ and $T_j\in F$ is a triangular face. $\mathcal{D}{\bf g}$ is defined similarly.
\end{theorem}

\smallskip

\begin{proof}
Denote ${\bf f} = \begin{pmatrix}
\mid & \mid & & \mid\\
f_1 & f_2 & \cdots & f_t\\
\mid & \mid & & \mid
\end{pmatrix}$, ${\bf g} = \begin{pmatrix}
\mid & \mid & & \mid\\
g_1 & g_2 & \cdots & g_t\\
\mid & \mid & & \mid
\end{pmatrix}$, $\mathcal{L}_1^{\mu} = \begin{pmatrix}
\mid & \mid & & \mid\\
\mu_1 & \mu_2 & \cdots & \mu_t\\
\mid & \mid & & \mid
\end{pmatrix}$ and $\mathcal{L}_2^{\mu} = \begin{pmatrix}
\mid & \mid & & \mid\\
\nu_1 & \nu_2 & \cdots & \nu_t\\
\mid & \mid & & \mid
\end{pmatrix}$. For each $j$, $f_j$ and $g_j$ can be extended to $\mathbb{C}$, by letting $f_j$ and $g_j$ be the identity map outside the image domain $\Omega$. Without loss of generality, we can assume $f_j$ and $g_j$ are normalized quasiconformal maps associated to $\mu_j$ and $\nu_j$ respectively. If $\alpha>1$ and $0<p\leq 1$ satisfy $2< 2\alpha < 1+\frac{1}{k}$, then there exist a positive integer $C(k,\alpha)$ such that
\[
||\mathcal{D}_1 f_j - \mathcal{D}_1 g_j||_2 \leq C(k,\alpha) ||\mu_j -\nu_j||_q \ \text{ and }\  ||\mathcal{D}_2 f_j - \mathcal{D}_2 g_j||_2 \leq C(k,\alpha) ||\mu_j -\nu_j||_q.\] 

\noindent where $q = \frac{p\alpha}{\alpha -1}$. Note that all matrix norms are equivalent. There exists a positive constant $A$ such that $||\cdot ||_q \leq A||\cdot ||_2$. Hence, 
\[
\begin{split}
||\mathcal{D} {\bf f} - \mathcal{D} {\bf g}||_F & = \left(\sum_{j=1}^l ||\mathcal{D}_1 f_j - \mathcal{D}_1 g_j||_2^2 + ||\mathcal{D}_2 f_j - \mathcal{D}_2 g_j||_2^2\right)^{1/2}\\
& \leq \left(\sum_{j=1}^l \frac{2C(k,\alpha)}{A^2}||\mu_j - \mathcal{D}_1 \nu_j||_2^2 \right)^{1/2}\\
& = \sqrt{2}A C(k,\alpha)||\mathcal{L}^{\mu}_1-\mathcal{L}^{\mu}_2||_F <\sqrt{2} A C(k,\alpha)\epsilon
\end{split}
\] 
The second inequality follows by letting $C_1 =\sqrt{2} A C(k,\alpha)$. 

For the first inequality, note that $f_j$ and $g_j$ are both normalized quasiconformal map for $j=1,2,...,l$. Then,
\[
f_j = {\bf v} + \mathcal{S}\mathcal{D}_2 f_j\ \  \text{ and }\ \  g_j = {\bf v} + \mathcal{S}\mathcal{D}_2 g_j
\]
\noindent where ${\bf v} = \begin{pmatrix} v_1\\v_2\\ \vdots \\ v_n
\end{pmatrix} \in \mathbb{C}^n
$ is the position vector of all vertices of $\Omega$. $\mathcal{S} \in M_{n\times |F|}(\mathbb{C})$ is defined in such a way that for any ${\bf h}\in \mathbb{C}^{|F|}$, $(\mathcal{S}{\bf h})_k = \sum_{m=1}^{|F|} w_{km} ({\bf h})_k$, where $w_{km} = \frac{1}{\pi} \int_{T_m} \frac{1}{v_k - \tau} d\tau$ and $T_m$ is the $m$-th triangular face of $\Omega$. Thus, we have
\[
\begin{split}
    ||f_j- g_j || &= ||\mathcal{S}(\mathcal{D}_2 f_j) - \mathcal{S}(\mathcal{D}_2 g_j)||_2\\
    & \leq ||\mathcal{S}||_2 ||\mathcal{D}_2 f_j - \mathcal{D}_2 g_j||_2\\
    & \leq ||\mathcal{S}||_2 C(k,\alpha)||\mu_j -\nu_j||_q.
\end{split}
\]
\noindent We can conclude that $||{\bf f} - {\bf g}||_F \leq A||\mathcal{S}||_2 C(k,\alpha) ||\mathcal{L}^{\mu}_1-\mathcal{L}^{\mu}_2||_F <A||\mathcal{S}||_2 C(k,\alpha)\epsilon$. The first inequality follows by letting $C_1 = A||\mathcal{S}||_2 C(k,\alpha)$.
\end{proof}

\bigskip

Theorem \ref{thm: longitudinal deformation} states that two longitudinal deformations are close if their Beltrami descriptors are close to each others. Furthermore, their degrees of smoothness are similar if their Beltrami descriptors are close. In other words, the longitudinal deformation is stable under the pertubation of the descriptor. It is a crucial observation, so that the manipulation of longitudinal deformations through Beltrami descriptors is not sensitive to the error of the descriptors. On the other hand, to alleviate the issue of large storage requirement, $\mathcal{L}^{\mathcal{F}}$ can be used to replace $\mathcal{L}^{\mu}$. Theorem \ref{thm: longitudinal deformation} tells us the reconstruction error of the longitudinal deformation is small if $\mathcal{L}^{\mu}$ and $\mathcal{L}^{\mathcal{F}}$ are close to each others.

\subsection{Decomposition of normal and abnormal components.}

In this subsection, we will first explain why the decomposed low-rank and sparse parts would have the meaning as desired and how we can decompose a longitudinal deformation into normal and abnormal components. 

To decompose a series of deformations into the normal and abnormal components, we propose to apply the Low-Rank Sparse Matrix Pursuit. Since the normal deformation is defined as the periodic motion displayed on the footage, it should correspond to the low rank part of the matrix capturing the series of deformations. The repeating pattern over a period of duration should be captured by the low rank component. However, when abnormal perturbations occur, the rank will be affected. In particular, the rank of the deformation matrix increases. The abnormal component (or perturbations) should correspond to the sparse part of the deformation matrix. In order to extract the abnormal component from the normal component, decomposing the deformation matrix into low rank and sparse components is a natural strategy.

As a remark, the representation of deformations using the Beltrami descriptor is advantageous. Theoretically, each deformation can be represented by its Beltrami coefficient. Conversely, given a Beltrami coefficient, the associated deformation can be reconstructed by solving the Beltrami's equation. The Beltrami coefficient effectively measures the geometric distortion under the associated deformation. In particular, the associated deformation is guaranteed to be bijective when the magnitude of the Beltrami coefficient is strictly less than 1 everywhere. Thus, representing the deformation matrix using Beltrami coefficients is more robust to the numerical error incurred during the process of low rank and sparse decomposition. On the contrary, if the deformation matrix is represented by vector fields, the bijectivity has to be controlled by the Jacobian constraint, which is hard to enforce. As such, the low rank and sparse decomposition of the deformation matrix represented by vector fields usually yield unnatural deformations with self-overlaps. Nevertheless, the low rank and sparse decomposition of the Beltrami descriptors can effectively decompose the deformations into normal (periodic)  and abnormal (perturbations) motions without self-overlaps, which will be discussed later on.

Given a deformation descriptor $\mathcal{L}^{\mu}$, we assume $\mathcal{L}^{\mu}$ is composed of the normal deformation $\mathcal{N}$ and abnormal deformation $\mathcal{A}$. Normal deformation $\mathcal{N}$ is often characterized by repeating pattern. Mathematically, $\mathcal{N}$ can be regarded as periodic and hence it should be of low rank. On the other hand, the abnormal deformation often occurs at some particular region and time. Thus, $\mathcal{A}$ can be assumed to be sparse. As such, our problem can be formulated as finding $\mathcal{N}$ and $\mathcal{A}$ such that they minimize:
\begin{equation}\label{eq: complex relaxed rpca}
\min\limits_{\mathcal{N},\mathcal{A}} \| \mathcal{N} \|_{\ast} + \alpha \|\mathcal{A}\|_{1}  \text{,  subject to } \mathcal{L}^{\mu}=\mathcal{N}+\mathcal{A} \in \mathbb{C}^{mn \times t}
\end{equation}

The first term involves the nuclear norm, aiming to minimize the rank of $\mathcal{N}$. The second term aims to enhance the sparsity of $\mathcal{A}$. The optimization problem can be solved using the alternating minimization method with multiplier (ADMM) as in the real case with suitable modifications. We will describe it in details as follows.

The Augmented Lagrangian function can be written as
\begin{equation}\label{eq: complex aug lagrangian}
E(\mathcal{N}^{k},\mathcal{A}^{k},Z^{k};\mathcal{L}^{\mu}) = \|\mathcal{N}^{k}\|_{\ast} + \alpha \|\mathcal{A}^{k}\|_{1} - \langle Z^{k}, \mathcal{N}^{k}+\mathcal{A}^{k}-\mathcal{L}^{\mu} \rangle + \frac{\beta_{k}(N)}{2}\|\mathcal{N}^{k}+\mathcal{A}^{k}-\mathcal{L}^{\mu}\|^{2}_{2}
\end{equation}
with $\langle X,Y \rangle = \mathrm{real}(\mathrm{tr}(X^{*}Y))= \mathrm{real}(\mathrm{tr}(XY^{*}))$ and $\beta_{k}(N) = \min \left\{ (1.5)^{k}\frac{1.25}{\|\mathcal{L}^{\mu}\|_{2}}, (1.5)^{N}\frac{1.25}{\|\mathcal{L}^{\mu}\|_{2}}  \right\}$. This $\beta_{k}(N)$ is defined in this way to ensure the recovered abnormal deformation to be bijective. Details will be provided in later section. ADMM to solve the optimizaton can be written as the following iterative scheme:
\begin{equation}\label{eq: iter scheme}
\begin{cases}
\mathcal{N}^{k+1} = \argmin_{\mathcal{N} \in \mathbb{C}^{mn \times t}} E(\mathcal{N},\mathcal{A}^k,Z^k;\mathcal{L}^{\mu})\ \ \ (\mathcal{N}\text{-subproblem}) \\
\mathcal{A}^{k+1} = \argmin_{\mathcal{A} \in \mathbb{C}^{mn \times t}} E(\mathcal{N}^{k+1},\mathcal{A},Z^k;\mathcal{L}^{\mu})  \ \ (\mathcal{A}\text{-subproblem})\\
Z^{k+1} = Z^{k} - \beta_{k}(N)(\mathcal{N}^{k+1}+\mathcal{A}^{k+1}-\mathcal{L}^{\mu}) \\
\end{cases}
\end{equation}

We will now describe how each subproblems can be tackled. We begin by looking into the $\mathcal{A}$-subproblem. Some definitions are needed to help our explanation.

\bigskip

\begin{definition}\label{def: row norm}
For $A \in \mathbb{C}^{M \times N}$, define the norm
    \begin{equation}\label{eq: row norm}
        \|A\|_{1,2} = \sum_{i=1}^{M} \left( \sum_{j=1}^{N} |a_{ij}|^{2} \right)^{\frac{1}{2}}
    \end{equation}
\end{definition}

\smallskip

It can be easily seen that Equation (\ref{eq: row norm}) sums each row's $L^2$ norm, and it clearly defines a matrix norm as well. Now, the $\mathcal{A}$-subproblem can be solved via a modified Euclidean projection, as described in the following proposition.

\bigskip

\begin{lemma}\label{thm: sparse matrix}
Define $f: \mathbb{R}^{N \times 2} \to \mathbb{R}$ by
\begin{equation}\label{eq: sparse objective function}
    f(X) = \alpha \|X\|_{1,2} + \frac{1}{2}\|M - X \|_{2}^{2} 
\end{equation}
where $M$ is a matrix in $\mathbb{R}^{N \times 2}$. Then the minimiser $X^{\ast}$ of $f$ is given by
\begin{equation}\label{eq: sparse matrix eq}
    X^{\ast}_{j} = \left( 1 - \frac{\alpha}{|M_{j}|} \right)_{+}M_{j}
\end{equation}
where $X^{\ast}_{j}$ is the $j$-th row of $X^{\ast}$ and $M_{j}$ is the $j$-th row of $M$. $|M_{j}|$ is the usual $L_{2}$ vector norm and $( y )_{+} = \max \{ y, 0 \}$ for $y \in \mathbb{R}$.
\end{lemma}

\begin{proof}
Minimising Equation (\ref{eq: sparse objective function}) is equivalent to minimising each row of $X$. In particular, we must have 
\begin{equation}
    0 \in \partial \left( \alpha |X^{\ast}_{j}|_{2} + \frac{1}{2}|M_{j} - X^{\ast}_{j}|^{2}_{2} \right) \ \ \text{ for } j =1,2,...,N.
\end{equation}
For $\|X^{\ast}_{j}\|_{2}\neq 0$, we have

\begin{equation}\label{eq: shrink x row}
    X^{\ast}_{j} = M_{j} - \hat{P}_{\mathbb{D}^{\alpha}_{\infty}}(M_{j}) 
\end{equation}
Equation (\ref{eq: shrink x row}) indicates that $X^{\ast}_{j}$ is obtained by reducing the magnitude of $M_{j}$ by $\alpha$ while keeping the same direction. If $\|M^{\ast}_{j}\|_{2}=0$, then by calculating the subdifferential of Equation (\ref{eq: sparse objective function}), we get:
\begin{equation}\label{eq: sparse optimisation x row zero}
    0 \in \alpha \{ x + \frac{1}{\alpha}M_{j} \ | \ x \in \partial(|X^{\ast}_{j}|_{2}) \}
\end{equation}

Hence, $ 0 \in \alpha \{ g + \frac{1}{\alpha} M_{j} \ | \ \|g\|_{2} \leq 1 \}$. This implies $|M_{j}| \leq \alpha$. Iterating over each row, we arrive at Equation (\ref{eq: sparse matrix eq}).
\end{proof}
\bigskip

\begin{theorem}\label{prop: sparse}
For $\mathcal{N},\mathcal{A},Z,\mathcal{L}^{\mu} \in \mathbb{C}^{(mn) \times t}$, the solution to the $\mathcal{A}$-subproblem is
\begin{equation}\label{eq: complex update S}
\mathcal{A}^{k+1} = \frac{1}{\beta_{k}(N)}Z^{k} - \mathcal{N}^{k} + \mathcal{L}^{\mu} - \hat{P}_{\mathbb{D}^{\alpha / \beta_{k}(N)}_{\infty}}\left( \frac{1}{\beta} Z^{k} - \mathcal{N}^{k} + \mathcal{L}^{\mu} \right)
\end{equation}
where $\hat{P}_{\mathbb{D}^{\alpha / \beta_{k}(N)}_{\infty}}$ denotes the Euclidean projection onto $\mathbb{D}^{\alpha / \beta_{k}(N)}_{\infty} := \{ z \in \mathbb{C} \ | \ |z| \leq \alpha / \beta_{k}(N) \}$.
\end{theorem}

\smallskip

\begin{proof}
To find the minimizer for the $\mathcal{A}$-subproblem, it is equivalent to solving
\begin{equation}\label{eq: alternating sparse}
\begin{aligned}
\mathcal{A}^{k+1} 
&= \argmin_{\mathcal{A}} \alpha \| \mathcal{A}\|_{1} + \frac{\beta_{k}(N)}{2}\|\mathcal{N}^{k}+\mathcal{A}-\mathcal{L}^{\mu}-\frac{1}{\beta_{k}(N)}Z^{k}\|^{2}_{2}\\
\end{aligned}
\end{equation}
Let $\varphi: \mathbb{C}^{(mn) \times t} \to \mathbb{R}^{(mnt) \times 2}$ be the transformation defined by:
\begin{equation}\label{eq: vectorise complex matrix isomorphism}
   \varphi (X) = \begin{pmatrix}
   \mathrm{Re}(\mathrm{vec}(X)), \mathrm{Im}(\mathrm{vec}(X))
   \end{pmatrix}
\end{equation}
where $\mathrm{vec}: \mathbb{C}^{p \times q} \to \mathbb{C}^{pq}$ is the function stacking columns of the input matrix into one vector. According to \cite{chan_2012}, Equation (\ref{eq: alternating sparse}) is indeed equivalent to 
\begin{equation}\label{eq: alternating sparse equivalent}
\mathcal{A}^* = \argmin_{\mathcal{A}} \frac{\alpha}{\beta_{k}(N)} \| \varphi(\mathcal{A}) \|_{1,2} + \frac{1}{2} \|\varphi(\mathcal{N}^{k}) + \varphi(\mathcal{A}) - \varphi(\mathcal{L}^{\mu}) - \frac{1}{\beta_{k}(N)} \varphi(Z^{k}) \|_{F}^{2}
\end{equation}
where $\| X \|_{1,2} = \sum_{j=1}^{n} \|X_{j}\|_{2}$ with $X_{j}$ denotes the $j$-th row of $X$. 

Putting $M = \frac{1}{\beta_{k}(N)}Z^{k} - \mathcal{N}^{k} + \mathcal{L}^{\mu}$, according to Lemma \ref{thm: sparse matrix}, it follows that with $\mathcal{A}^{\ast}_{j}$ and $M_{j}$ denoting the $j$-th row of $\mathcal{A}$ and $M$ respectively,
\begin{equation}
\mathcal{A}^*_{j} = \left( 1 - \frac{\alpha}{\beta_{k}(N)}\frac{1}{|M_{j}|} \right)_{+}M_{j}
\end{equation}
If $|M_{j}| \leq \frac{\alpha}{\beta_{k}(N)}$, $\mathcal{A}^*_{j} = 0$. If $|M_{j}| > \frac{\alpha}{\beta_{k}(N)}$, 
\begin{equation}
    \begin{split}
        \mathcal{A}^{\ast}_{j} & = \left( 1 - \frac{\alpha}{\beta_{k}(N)}\frac{1}{|M_{j}|} \right) M_{j}\\
        & = M_{j} - \hat{P}_{\mathbb{D}^{\alpha / \beta_{k}(N)}_{\infty}}(M_{j})
    \end{split}
\end{equation}
Then Formula (\ref{eq: complex update S}) follows.
\end{proof}

\bigskip

Theorem (\ref{eq: alternating sparse}) is important as it gives us a closed form solution to solve the $\mathcal{A}$-subproblem during the ADMM iteration.

Next, we will look at the $\mathcal{N}$-subproblem. Indeed, the $\mathcal{N}$-subproblem can be treated exactly as in the real case, which is described as follows.

\smallskip

\begin{theorem}\label{corollary: low rank}
For $\mathcal{N},\mathcal{A},Z,\mathcal{L}^{\mu} \in \mathbb{C}^{mn \times t}$, the solution to the low-rank subproblem in Equation (\ref{eq: update L}) is
\begin{equation}\label{eq: complex update L}
   \mathcal{N}^{k+1} = U^{k+1}\mathrm{diag} \left( \max\{ \sigma_{i}^{k+1} - \frac{1}{\beta_{k}(N)},0 \} \right) (V^{k+1})^{T}
\end{equation}
where $U^{k+1},V^{k+1},\sigma_{i}^{k+1}$ are obtained by SVD that is:
\begin{equation}\label{eq: svd update 2}
   \mathcal{L}^{\mu} - \mathcal{A}^{k+1} + \frac{1}{\beta_{k}(N)} Z^{k} = U^{k+1}\Sigma^{k+1}(V^{k+1})^{T} \quad \text{with} \quad \Sigma^{k+1} = \mathrm{diag} \left( \{ \sigma_{i}^{k+1} \}_{i=1}^{r} \right) 
\end{equation}
\end{theorem}

The proof of the above theorem follows similarly as in the case of real-valued matrices. We refer readers to \cite{ma_2009,cai_2008} for the details of the proof.

It is worth mentioning that different literature has provided theoretical guarantee that the ADMM approach on this RPCA will converge. Readers can refer to \cite{hong_2012,yuan_2009,gabay_1976,gabay_1983,glowinski_2013,glowinsku_1989,he_1998,ye_2007}. In particular, Hong et. al.\cite{hong_2012} proved that the approach has linear time convergence.  

We summarize the algorithm for the decomposition of $L^{\mu}$ into normal and abnormal deformations as follows. \par
\begin{algorithm}[H]\label{algo: rpca}
\SetAlgoLined
\SetKwInOut{Input}{Input}
\SetKwInOut{Output}{Output}
\SetKwInOut{Initialisation}{Initialisation}
\Input{$\mathcal{L}^{\mu} \in \mathbb{C}^{mn \times t}$, $N\in \mathbb{N}$}
\Output{Normal component $\mathcal{N}$ and abnormal component $\mathcal{A}$}
\Initialisation{$\mathcal{N}_{0}$ be a zero matrix, $Z_{0} = \mathcal{L}^{\mu} / \|\mathcal{L}^{\mu}\|_{2}$, $\beta_{k}(N) = \min \{ (1.5)^{k} \frac{1.25}{\| \mathcal{L}^{\mu} \|_{2}}, (1.5)^{N} \frac{1.25}{\| \mathcal{L}^{\mu} \|_{2}}\}$}
\While{not converge}{
Update $\mathcal{N}^{k+1}$ using Equation (\ref{eq: complex update L}) \;
Update $\mathcal{A}^{k+1}$ using Equation (\ref{eq: complex update S}) \;
$Z_{k+1} \gets Z_{k} + \beta_{k}(n)(\mathcal{L}^{\mu}-\mathcal{N}^{k+1}-\mathcal{A}^{k+1})$
}
\caption{Decomposition of $\mathcal{L}^{\mu}$}
\end{algorithm}

\bigskip

Here, $N$ is a chosen integer parameter.
Once $\mathcal{L}^{\mu}$ is decomposed into $\mathcal{N}$ and $\mathcal{A}$, the associated normal and abnormal longitudinal deformations can be reconstructed according to Equation (\ref{LBS}). \par

The subtle perturbation from a longitudinal deformation are supposedly bijective without overlaps. A crucial question is whether our extracted abnormal deformation is indeed bijective. As a matter of fact, performing the low rank and sparse decomposition on the Beltrami descriptor is beneficial, since we can theoretically guarantee the bijectivity of the extracted abnormal deformation under suitable choice of the parameter. Hence, our algorithm can give a realistic and accurate extracted component for further deformation analysis. This fact is explained in details with the following theorem.

\begin{theorem}\label{thm: bijectivity of decomp}
Considering Equation (\ref{eq: complex aug lagrangian}), there exists a constant $c(\mathcal{L}^{\mu})$ such that if 
\begin{equation}\label{eq: alpha condition}
   \alpha > c(\mathcal{L}^{\mu}) = \frac{\| \mathcal{L}^{\mu} \|_{\max}}{\| \mathcal{L}^{\mu} \|_{2}} + \frac{1.25p}{\| \mathcal{L}^{\mu} \|_{2}}\frac{1 - (1.5)^{N}q^{N}}{1-1.5q} + \frac{\beta_{N}(N)pq^{N}}{1-q}
\end{equation}
where $p,q$ depend on $\mathcal{L}^{\mu}$ and $\| \mathcal{L}^{\mu} \|_{\max} < 1$, then our algorithm \ref{algo: rpca} would yield $\| \mathcal{A}^{k} \|_{\max} < 1$ for all $k \in \mathbb{N}$.
\end{theorem}
\begin{proof}
The proof is based on induction on $k$. We first check the base case
\begin{equation}\label{eq: MI base case}
\begin{aligned}
    \| \mathcal{A}^{1} \|_{\max} &= \left\| \frac{1}{\beta_{0}}Z^{0} + \mathcal{L}^{\mu} - \hat{P}_{\mathbb{D}^{\alpha / \beta_{0}}_{\infty}} \left( \frac{1}{\beta_{0}}Z^{0} + \mathcal{L}^{\mu} \right) \right\|_{\max} < 1 \\
    \iff \left\| \frac{1}{\beta_{0}}Z^{0} + \mathcal{L}^{\mu} \right\|_{\max} &< 1 + \frac{\alpha}{\beta_{0}}
\end{aligned}
\end{equation}
Clearly, $\| \mathcal{L}^{\mu} \|_{\max} < 1$ and since $Z^{0}$ is defined to be $\mathcal{L}^{\mu} / \| \mathcal{L}^{\mu} \|_{2}$, we proved the base case. \par
Assume it is true that $\| \mathcal{A}^{k} \|_{\max} < 1$. From \cite{hong_2012} by Hong et. al., we have, for some constant $p > 0$, $q \in (0,1)$
\begin{equation}\label{eq: R linearity}
    \| \mathcal{L}^{\mu} - \mathcal{N}^{r} - \mathcal{A}^{r} \| \leq pq^{r}
\end{equation}
which is known as the R-linearity of convergence of ADMM. Note that using Equation (\ref{eq: complex update S}) in Formula (\ref{eq: MI base case})
\begin{equation}\label{eq: k+1 case equivalence}
    \| \mathcal{A}^{k+1} \|_{\max} < 1 \iff \left\|\frac{1}{\beta_{k}(N)}Z^{k} - \mathcal{N}^{k} + \mathcal{L}^{\mu} \right\|_{\max} < 1 + \frac{\alpha}{\beta_{k}(N)}
\end{equation}
Considering this specific term, we deduce that
\begin{equation}\label{eq: induction process}
\begin{aligned}
     &\left\|\frac{1}{\beta_{k}(N)}Z^{k} - \mathcal{N}^{k} + \mathcal{L}^{\mu} \right\|_{\max} \\
     &= \left\|  \frac{1}{\beta_{k}(N)} \left( \frac{\mathcal{L}^{\mu}}{\| \mathcal{L} \|_{2}} + \sum_{i=1}^{k-1} \beta_{i}(N) \left( \mathcal{L}^{\mu} - \mathcal{N}^{i} - \mathcal{A}^{i} \right) \right) - \frac{\beta_{k}(N)}{\beta_{k}(N)}\left( \mathcal{L}^{\mu} - \mathcal{N}^{k} - \mathcal{A}^{i} \right) + \mathcal{A}^{k} \right\|_{\max} \\
     &\leq \frac{1}{\beta_{k}(N)}\left\| \frac{\mathcal{L}^{\mu}}{\| \mathcal{L} \|_{2}} + \sum_{i=1}^{k}  \beta_{i}(N)(\mathcal{L}^{\mu} - \mathcal{N}^{i} - \mathcal{A}^{i})\right\|_{\max} +  \|\mathcal{A}^{k}\|_{\max} \\
\end{aligned}
\end{equation}
Notice that using Equation (\ref{eq: R linearity}), we have
\begin{equation}\label{eq: beta term}
\begin{aligned}
    &\left\| \frac{\mathcal{L}^{\mu}}{\| \mathcal{L} \|_{2}} + \sum_{i=1}^{k}  \beta_{i}(N)(\mathcal{L}^{\mu} - \mathcal{N}^{i} - \mathcal{A}^{i})\right\|_{\max} \\
    &= \frac{\| \mathcal{L}^{\mu} \|_{\max}}{\| \mathcal{L}^{\mu} \|_{2}} + \sum_{i=1}^{n-1} \beta_{i}(N)\| \mathcal{L}^{\mu} - \mathcal{N}^{i}- \mathcal{A}^{i}\|_{\max} + \sum_{i=n}^{k} \beta_{i}(N)\|\mathcal{L}^{\mu} - \mathcal{N}^{i}- \mathcal{A}^{i}\|_{\max} \\
    &\leq \frac{\| \mathcal{L}^{\mu} \|_{\max}}{\| \mathcal{L}^{\mu} \|_{2}} + \sum_{i=1}^{n-1} \beta_{i}(N)pq^{i} + \sum_{i=n}^{k} \beta_{i}(N)pq^{i} \\
    &= \frac{\| \mathcal{L}^{\mu} \|_{\max}}{\| \mathcal{L}^{\mu} \|_{2}} + \frac{1.25p}{\| \mathcal{L}^{\mu} \|_{2}}\frac{1 - (1.5)^{N}q^{N}}{1-1.5q} + \frac{\beta_{N}(N)pq^{N}}{1-q} \\
    &< \alpha
\end{aligned}
\end{equation}
Thus, putting the last Inequality (\ref{eq: beta term}) into Equation (\ref{eq: induction process}), we have
\begin{equation}\label{eq: MI conclusion}
    \left\| \frac{1}{\beta_{k}(N)}Z^{k} - \mathcal{N}^{k} + \mathcal{L}^{\mu} \right\|_{\max} < 
\frac{\alpha}{\beta_{k}(N)} + 1
\end{equation}
which implies that $\| \mathcal{A}^{k+1} \|_{\max} < 1$, which completes the inductive proof. 
\end{proof}

With all tools introduced, we can now describe our whole algorithm as follows. \par
\begin{algorithm}[H]\label{algo: main}
\SetAlgoLined
\SetKwInOut{Input}{Input}
\SetKwInOut{Output}{Output}
\Input{Reference frame $I_{ref}$, and video frame $\{I_{i}\}_{i=1}^{t}$}
\Output{Low-rank frames $\{\mathbf{l}_{i}\}_{i=1}^{t}$ and sparse frames $\{\mathbf{s}_{i}\}_{i=1}^{t}$}
\For{each frame $I_{i}$(parallel-computation compatible)}{
Register $I_{ref}$ to $I_{i}$ get the deformation field\;
Compute the Beltrami descriptor $\mathcal{L}^{\mu}$\;
}
Using algorithm \ref{algo: rpca}, decompose $\mathcal{L}^{\mu}=\mathcal{N}+\mathcal{A}$\;
\For{each column $l_{i}$ of $\mathcal{N}$(parallel-computation compatible)}{
Using LBS, recover $l_{i}$ to a map $f^{\mathcal{N}}_{i}$ \;
Deform $I_{ref}$ with the map $f^{\mathcal{N}}_{i}$ and we obtain $\mathbf{l}_{i}$ \;
}
\For{each column $s_{i}$ of $\mathcal{A}$(parallel-computation compatible)}{
Using LBS, recover $s_{i}$ to a map $f^{\mathcal{A}}_{i}$ \;
Deform $I_{ref}$ with the map $f^{\mathcal{A}}_{i}$ and we obtain $\mathbf{s}_{i}$\;
}
\caption{Abnormal Deformation Extraction and Recovery Algorithm}
\end{algorithm}

\bigskip

Finally, we remark that the robustness actually depends on the accuracy of the registration. For a large deformation between the reference frame and the image frame, some registration methods may not yield accurate registration result. As a result, the decomposition of the series of deformations may become inaccurate. Nevertheless, if the deformation across all frames are not very large, we can achieve similar performance whichever frame is chosen as the reference frame. For the case with large deformations across frames, we usually choose a frame, which is least deviated from every frame, as the reference image. With this choice, our method can still achieve very accurate decomposition effectively.


\section{Experimental Result}\label{sec: experiment}

In this section, we present our experimental results on synthetic images, as well as on real medical images.\par

\bigskip

\noindent {\bf Example 1:} We first test our proposed method on a synthetic image sequence. The input data is a sequence of binary images that shows a circle shrinks and expands, and repeats this process for a few cycles. Readers can refer to Figure \ref{fig: circle cycle} to visualise this process. The total number of frame of this process is 48, which means that the ground-truth rank of the Fourier Transformed BC matrix is 24. 
\begin{figure}[h]
\centering
\includegraphics[scale=.75]{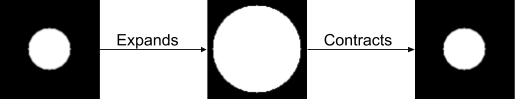}
\caption{Illustration of a Normal Cycle of the First Synthetic Experiment}
\label{fig: circle cycle}
\end{figure}
The whole expansion and contraction process is repeated 9 times, and 3 of which are perturbed by adding some deformations around the boundary of the circle. After adding perturbation on the cycles, the rank of the Beltrami descriptor matrix raises to 47, while our algorithm successfully reduces the rank of the low-rank matrix to 27. We remark that since we took the smallest circle as the reference image, one can observe that the recovered sparse image has a circle that is far smaller than perturbed frames. Table \ref{table: circle pics} shows the result of our algorithm on one of the three perturbations. \par

\begin{table}[t]
\centering
\begin{tabular}{|l|l|l|l|l|l|}
\hline
\begin{tabular}[c]{@{}l@{}}Original \\ Frame \end{tabular} & \begin{tabular}[c]{@{}l@{}}Perturbed \\ Frame \end{tabular} & 
\begin{tabular}[c]{@{}l@{}}Recovered \\ Low-Rank \\ Frame \end{tabular}&\begin{tabular}[c]{@{}l@{}}Recovered \\ Sparse\\ Frame \end{tabular} &
\begin{tabular}[c]{@{}l@{}}Recovered \\ Low-Rank \\ Frame \\ on Vector \\ Field \end{tabular}&\begin{tabular}[c]{@{}l@{}}Recovered \\ Sparse\\ Frame \\ on Vector \\ Field\end{tabular} \\ \hline
\includegraphics[scale=.5]{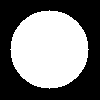}& \includegraphics[scale=.5]{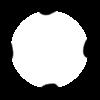}&\includegraphics[scale=.5]{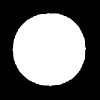}&\includegraphics[scale=.5]{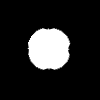}&
\includegraphics[scale=.5]{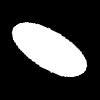}&\includegraphics[scale=.5]{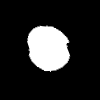}\\ \hline
\includegraphics[scale=.5]{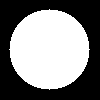}& \includegraphics[scale=.5]{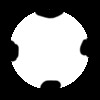}&\includegraphics[scale=.5]{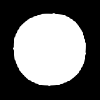}&\includegraphics[scale=.5]{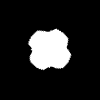}&
\includegraphics[scale=.5]{fig/results/circle/vec_without_fft/low_rank/306.png}&\includegraphics[scale=.5]{fig/results/circle/vec_without_fft/sparse/306.png}\\ \hline
\includegraphics[scale=.5]{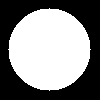}& \includegraphics[scale=.5]{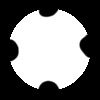}&\includegraphics[scale=.5]{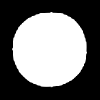}&\includegraphics[scale=.5]{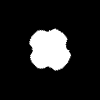}&
\includegraphics[scale=.5]{fig/results/circle/vec_without_fft/low_rank/306.png}&\includegraphics[scale=.5]{fig/results/circle/vec_without_fft/sparse/306.png}\\ \hline
\includegraphics[scale=.5]{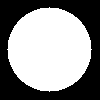}& \includegraphics[scale=.5]{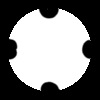}&\includegraphics[scale=.5]{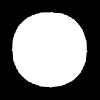}&\includegraphics[scale=.5]{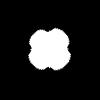}&
\includegraphics[scale=.5]{fig/results/circle/vec_without_fft/low_rank/306.png}&\includegraphics[scale=.5]{fig/results/circle/vec_without_fft/sparse/306.png}\\ \hline
\includegraphics[scale=.5]{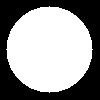}& \includegraphics[scale=.5]{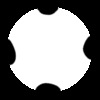}&\includegraphics[scale=.5]{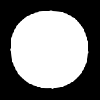}&\includegraphics[scale=.5]{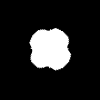}&
\includegraphics[scale=.5]{fig/results/circle/vec_without_fft/low_rank/306.png}&\includegraphics[scale=.5]{fig/results/circle/vec_without_fft/sparse/306.png}\\ \hline
\includegraphics[scale=.5]{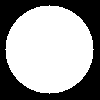}& \includegraphics[scale=.5]{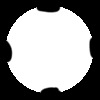}&\includegraphics[scale=.5]{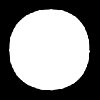}&\includegraphics[scale=.5]{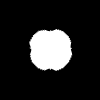}&
\includegraphics[scale=.5]{fig/results/circle/vec_without_fft/low_rank/306.png}&\includegraphics[scale=.5]{fig/results/circle/vec_without_fft/sparse/306.png}\\ \hline
\includegraphics[scale=.5]{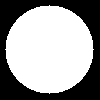}& \includegraphics[scale=.5]{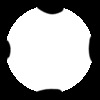}&\includegraphics[scale=.5]{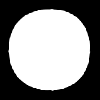}&\includegraphics[scale=.5]{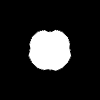}&
\includegraphics[scale=.5]{fig/results/circle/vec_without_fft/low_rank/306.png}&\includegraphics[scale=.5]{fig/results/circle/vec_without_fft/sparse/306.png}\\ \hline
\end{tabular}
\caption{Results of Example 1}
\label{table: circle pics}
\end{table}

A straightforward method to decompose the longitudinal deformation is done by applying the RPCA method on the vector fields of the deformation. As mentioned, vector fields cannot effectively capture the geometric information of the deformation. As such, RPCA method on vector fields cannot yield satisfactory results. The last two columns of Table \ref{table: circle pics} show the results of pursuing the low-rank and sparse part on the deformation vector fields obtained from registering the reference frame to each of the video frames. We view each vector in the vector fields as an element in $\mathbb{C}$, and we stacked them horizontally and obtain a giant matrix, with each column corresponds to the deformation vector field from the referenced frame to each frame defined at each pixel. Then we run the complex matrix decomposition algorithm on this matrix. Although the decomposed low-rank matrix is of rank 24, the last two columns of Table \ref{table: circle pics} clearly shows that the recovered results are far from the ground-truth to be useful: The circles are distorted to ellipses, which they should not be. Compared to the results obtained from our original longitudinal deformation descriptor, this decomposition is not meaningful.\par

\begin{table}[t]
\centering
\begin{tabular}{|l|l|l|l|}
\hline
\begin{tabular}[c]{@{}l@{}} Original \\ Frame \end{tabular} & \begin{tabular}[c]{@{}l@{}}Perturbed \\ Frame \end{tabular}& 
\begin{tabular}[c]{@{}l@{}}Recovered \\ Low-Rank \\ Frame \end{tabular}&\begin{tabular}[c]{@{}l@{}}Recovered \\ Sparse\\ Frame \end{tabular} \\ \hline
\includegraphics[scale=.75]{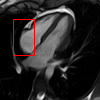}& \includegraphics[scale=.75]{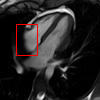}&\includegraphics[scale=.75]{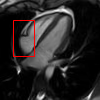}&\includegraphics[scale=.75]{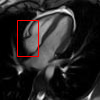}\\ \hline
\includegraphics[scale=.75]{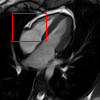}& \includegraphics[scale=.75]{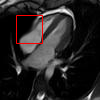}&\includegraphics[scale=.75]{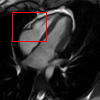}&\includegraphics[scale=.75]{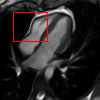}\\ \hline
\includegraphics[scale=.75]{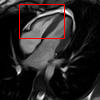}& \includegraphics[scale=.75]{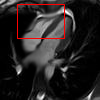}&\includegraphics[scale=.75]{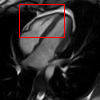}&\includegraphics[scale=.75]{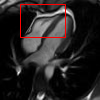}\\ \hline
\includegraphics[scale=.75]{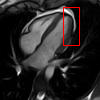}& \includegraphics[scale=.75]{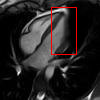}&\includegraphics[scale=.75]{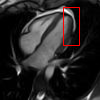}&\includegraphics[scale=.75]{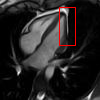}\\ \hline
\end{tabular}
\caption{Results of Example 2}
\label{table: first heart pics}
\end{table}

\bigskip

\noindent {\bf Example 2:} The next example is on a sequence of real medical images of a beating heart. The original video contains 341 frames with repeated periodic beating of 31 times. In this example, artificial abnormal deformations are introduced to one of the 31 cycles, and so ground-truth images are available to study the accuracy of our proposed model. Table \ref{table: first heart pics} shows the result. The second column shows the frames with manual deformation added on images in the first column, and the red box area is where deformation is added. We can see that our algorithm can almost recover the low-rank frames to the original frames and the sparse frames to the perturbed frames. The size of the input Beltrami descriptor is $19602 \times 341$, and the rank of the original video and perturbed video are 11 and 15 respectively. After running our algorithm on the matrix, the rank of the recovered low-rank matrix is reduced to 11. \par
Beside the recovered rank, from Table \ref{table: first heart pics}, we can see that our algorithm can capture and recover both the normal and abnormal deformation on the beating heart to great details. It can be seen that the recovered low-rank frames looks very much alike to the original frames and recovered sparse frames can effective capture the abnormal deformation. \par

\begin{figure}
    \begin{subfigure}{.5\textwidth}
    \centering
    \includegraphics[width=.6\linewidth]{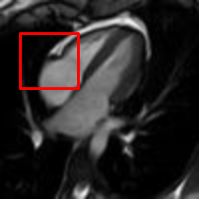}
    \caption{Ground Truth Frame}
    \label{subfig: original} 
    \end{subfigure}
    \begin{subfigure}{.5\textwidth}
    \centering
    \includegraphics[width=.6\linewidth]{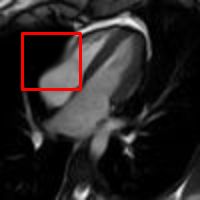}
    \caption{Perturbed Frame}
    \label{subfig: perturbed} 
    \end{subfigure}
    
    \begin{subfigure}{.5\textwidth}
    \centering
    \includegraphics[width=.6\linewidth]{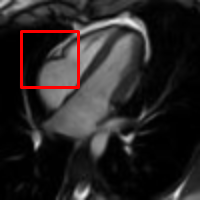}
    \caption{Recovered Low-rank Frame}
    \label{subfig: low rank} 
    \end{subfigure}
    \begin{subfigure}{.5\textwidth}
    \centering
    \includegraphics[width=.6\linewidth]{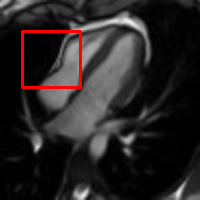}
    \caption{Recovered Sparse Frame}
    \label{subfig: sparse} 
    \end{subfigure}
    \caption{Second Row of Table \ref{table: first heart pics}}
    \label{fig: heart 1 zoom in pics}
\end{figure}
To better observe the result of our experiment, Figure \ref{fig: heart 1 zoom in pics} shows the second row of Table \ref{table: first heart pics}. The area bounded by the red box is the periphery of the beating heart. Figure \ref{subfig: original} shows the ground truth frame. Figure \ref{subfig: perturbed} shows a frame with a perturbation combined with the normal deformation. The deformation between a frame and the reference image is computed by registration and represented by the Beltrami coefficient. Figure \ref{subfig: low rank} shows the deformed image from the reference frame by the low rank part of the deformation. It closely resemble the ground truth frame as shown in Figure \ref{subfig: original}. Figure \ref{subfig: sparse} shows the deformed image from the reference frame by the sparse part of the deformation. It demonstrates how the abnormal motion deforms the image from the reference frame. Thus, Figure \ref{subfig: sparse} should be different from Figure \ref{subfig: perturbed}, since Figure \ref{subfig: perturbed} combines both the normal and abnormal motions.

\begin{figure}
    \begin{subfigure}{.5\textwidth}
    \centering
    \includegraphics[scale = .4]{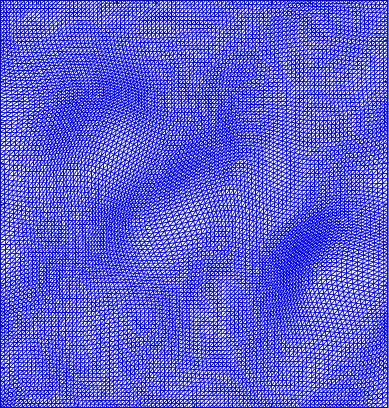}
    \caption{GT Low-Rank Deformation}
    \label{subfig: grid gt low rank}
    \end{subfigure}
    \begin{subfigure}{.5\textwidth}
    \centering
    \includegraphics[scale = .4]{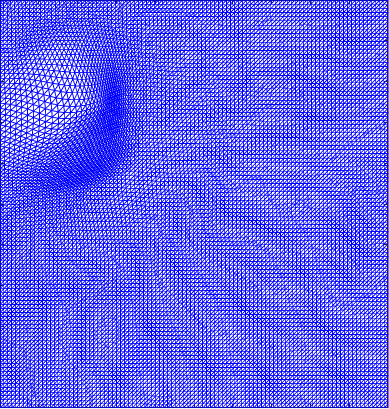}
    \caption{GT Sparse Deformation}
    \label{subfig: grid gt sparse}
    \end{subfigure}
    
    \begin{subfigure}{.5\textwidth}
    \centering
    \includegraphics[scale = .4]{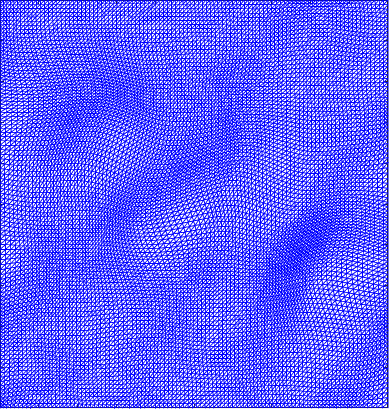}
    \caption{Recovered Low-Rank Deformation}
    \label{subfig: grid recovered low rank}
    \end{subfigure}
    \begin{subfigure}{.5\textwidth}
    \centering
    \includegraphics[scale = .4]{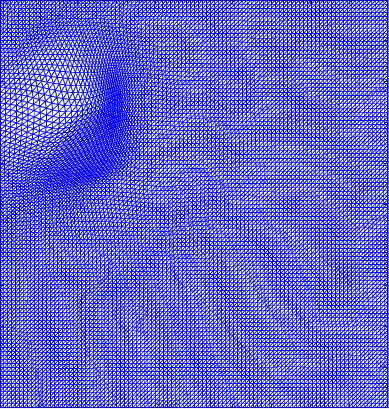}
    \caption{Recovered Sparse Deformation}
    \label{subfig: grid recovered sparse}
    \end{subfigure}
    \caption{Comparison Between Ground-Truth Maps and Recovered Maps}
    \label{fig: heart 1 map comparison}
\end{figure}
Figure \ref{fig: heart 1 map comparison} shows the visualisation of Figure \ref{fig: heart 1 zoom in pics} mappings in the form of grids. Let $\mu_{1}$ and $\mu_{2}$ be the Beltrami coefficients of the registration maps from the reference frame to the ground truth frame in Figure \ref{subfig: original} and the perturbed frame in Figure \ref{subfig: perturbed} respectively. Figure \ref{subfig: grid gt low rank} shows the mapping associated to Beltrami coefficient $\mu_{1}$, and Fig \ref{subfig: grid gt sparse} shows the mapping associated to BC $\mu_{2} - \mu_{1}$. From Figure \ref{subfig: grid recovered low rank} and Fig \ref{subfig: grid recovered sparse}, we can see that our method successfully restored the normal and abnormal deformation. Figure \ref{fig: heart 1 map comparison} serves as evidence that our decomposition is meaningful, in the sense that our method does not blindly return a Beltrami Descriptor with certain periodicity, but the decomposed descriptor does carry our desired information to recover the deformation to a large extent.

\begin{table}[t]
\centering
\begin{tabular}{|l|l|l|l|}
\hline
\begin{tabular}[c]{@{}l@{}} Frame \\ without abnormal \\ deformation\end{tabular} & \begin{tabular}[c]{@{}l@{}}Frame \\ with abnormal \\ deformation\end{tabular}& 
\begin{tabular}[c]{@{}l@{}}Recovered\\ Low-Rank \\ Frame \end{tabular}&\begin{tabular}[c]{@{}l@{}}Recovered\\ Sparse\\ Frame \end{tabular} \\ \hline
\includegraphics[scale=.75]{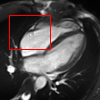}& \includegraphics[scale=.75]{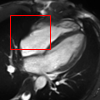}&\includegraphics[scale=.75]{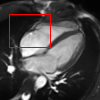}&\includegraphics[scale=.75]{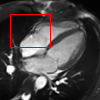}\\ \hline
\includegraphics[scale=.75]{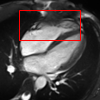}& \includegraphics[scale=.75]{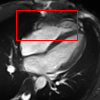}&\includegraphics[scale=.75]{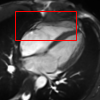}&\includegraphics[scale=.75]{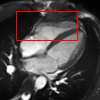}\\ \hline
\includegraphics[scale=.75]{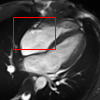}& \includegraphics[scale=.75]{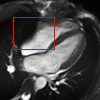}&\includegraphics[scale=.75]{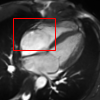}&\includegraphics[scale=.75]{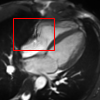}\\ \hline
\includegraphics[scale=.75]{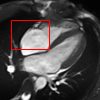}& \includegraphics[scale=.75]{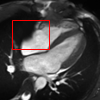}&\includegraphics[scale=.75]{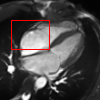}&\includegraphics[scale=.75]{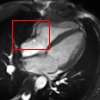}\\ \hline
\includegraphics[scale=.75]{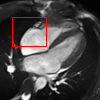}& \includegraphics[scale=.75]{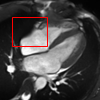}&\includegraphics[scale=.75]{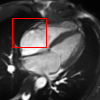}&\includegraphics[scale=.75]{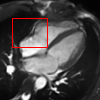}\\ \hline
\includegraphics[scale=.75]{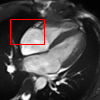}& \includegraphics[scale=.75]{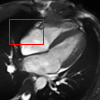}&\includegraphics[scale=.75]{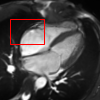}&\includegraphics[scale=.75]{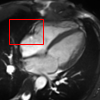}\\ \hline
\end{tabular}
\caption{Result of Example 3}
\label{table: second heart pics}
\end{table}

\bigskip

\noindent {\bf Example 3:} In this example, we test our algorithm on another medical video of a beating heart with abnormal perturbations. The original rank of the video is 36. After performing our proposed method on the Beltrami descriptor, the rank of the low-rank matrix is reduced to 20. Table \ref{table: second heart pics} displays one of the perturbation and its recovery by our method. As shown in the table, the results show that our algorithm can recover the normal and abnormal deformation. Readers can compare the first column with the third, and the second with the fourth. \par

\begin{figure}
    \begin{subfigure}{.5\textwidth}
    \centering
    \includegraphics[scale = .4]{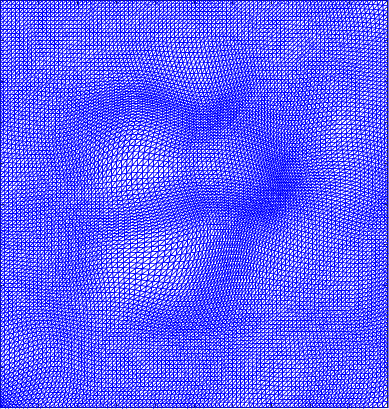}
    \caption{GT Low-Rank Deformation}
    \label{subfig: grid gt low rank second heart}
    \end{subfigure}
    \begin{subfigure}{.5\textwidth}
    \centering
    \includegraphics[scale = .4]{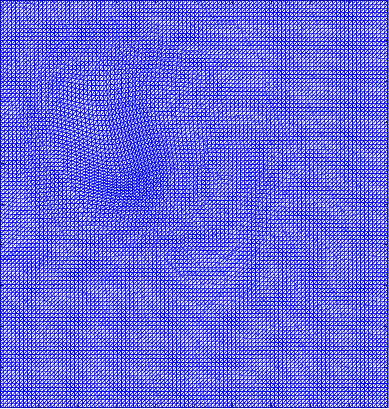}
    \caption{GT Sparse Deformation}
    \label{subfig: grid gt sparse second heart}
    \end{subfigure}
    
    \begin{subfigure}{.5\textwidth}
    \centering
    \includegraphics[scale = .4]{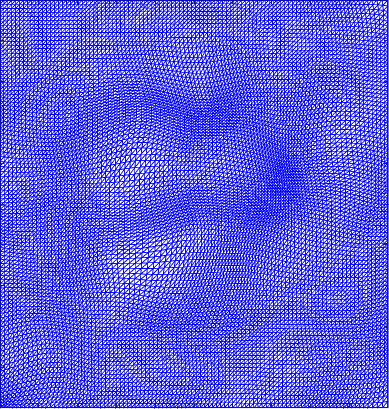}
    \caption{Recovered Low-Rank Deformation}
    \label{subfig: grid recovered low rank second heart}
    \end{subfigure}
    \begin{subfigure}{.5\textwidth}
    \centering
    \includegraphics[scale = .4]{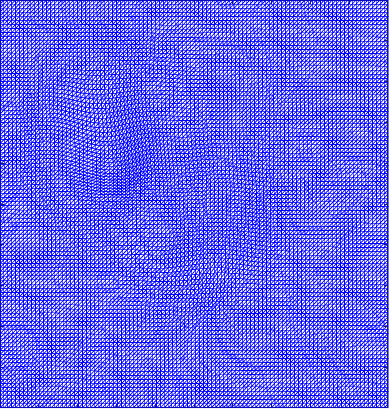}
    \caption{Recovered Sparse Deformation}
    \label{subfig: grid recovered sparse second heart}
    \end{subfigure}
    \caption{Comparison Between Ground-Truth Maps and Recovered Maps}
    \label{fig: heart 2 map comparison}
\end{figure}
Figure \ref{fig: heart 2 map comparison} shows the ground truth and recovered maps of the registration maps in the last row of Table \ref{table: first lung pics}. It can be seen that ground truth low-rank mapping shown in Figure \ref{subfig: grid gt low rank second heart} resembles the recovered low-rank mapping shown in Figure \ref{subfig: grid recovered low rank second heart}. This again shows that we could obtain a meaningful mapping from the decomposition.

\bigskip

\noindent {\bf Example 4:} In this example, we test our algorithm on another medical video of a lung under respiration. The original video capture 31 cycles with some perturbation at some frames. The rank of the input longitudinal Beltrami descriptor is 23, which was reduced to 10 after performing our algorithm on it. Table \ref{table: first lung pics} showed the pictures of one of the perturbation. \par
\begin{table}[t]
\centering
\begin{tabular}{|l|l|l|l|l|l|}
\hline
\begin{tabular}[c]{@{}l@{}} Original \\ Frame \end{tabular} & \begin{tabular}[c]{@{}l@{}}Perturbed \\ Frame \end{tabular}& 
\begin{tabular}[c]{@{}l@{}}Recovered \\ Low-Rank \\ Frame \end{tabular}&\begin{tabular}[c]{@{}l@{}}Recovered \\ Sparse\\ Frame \end{tabular}&
\begin{tabular}[c]{@{}l@{}} Recovered \\ Sparse\\ Frame on \\ Vector \\ Field \\ with \\ FFT \end{tabular}& 
\begin{tabular}[c]{@{}l@{}} Recovered \\ Sparse \\ Frame on \\ Vector \\ Field \\ without \\ FFT\end{tabular} \\ \hline
\includegraphics[scale=.5]{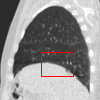}& \includegraphics[scale=.5]{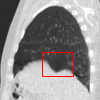}&\includegraphics[scale=.5]{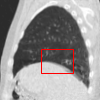}&\includegraphics[scale=.5]{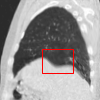}&\includegraphics[scale=.5]{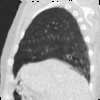}&\includegraphics[scale=.5]{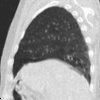}\\ \hline
\includegraphics[scale=.5]{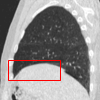}& \includegraphics[scale=.5]{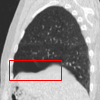}&\includegraphics[scale=.5]{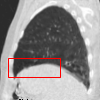}&\includegraphics[scale=.5]{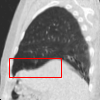}&\includegraphics[scale=.5]{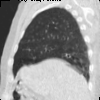}&\includegraphics[scale=.5]{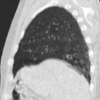}\\ \hline
\includegraphics[scale=.5]{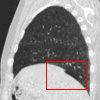}& \includegraphics[scale=.5]{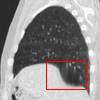}&\includegraphics[scale=.5]{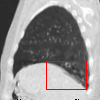}&\includegraphics[scale=.5]{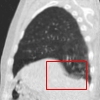}&\includegraphics[scale=.5]{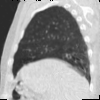}&\includegraphics[scale=.5]{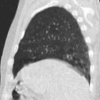}\\ \hline
\includegraphics[scale=.5]{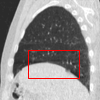}& \includegraphics[scale=.5]{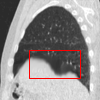}&\includegraphics[scale=.5]{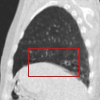}&\includegraphics[scale=.5]{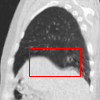}&\includegraphics[scale=.5]{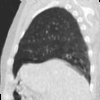}&\includegraphics[scale=.5]{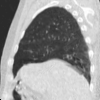}\\ \hline
\includegraphics[scale=.5]{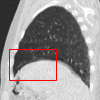}& \includegraphics[scale=.5]{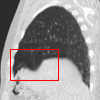}&\includegraphics[scale=.5]{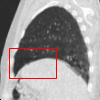}&\includegraphics[scale=.5]{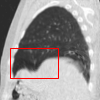}&\includegraphics[scale=.5]{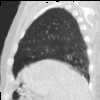}&\includegraphics[scale=.5]{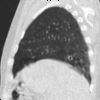}\\ \hline
\includegraphics[scale=.5]{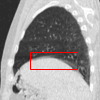}& \includegraphics[scale=.5]{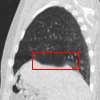}&\includegraphics[scale=.5]{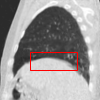}&\includegraphics[scale=.5]{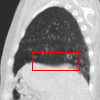}&\includegraphics[scale=.5]{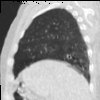}&\includegraphics[scale=.5]{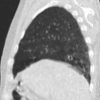}\\ \hline
\includegraphics[scale=.5]{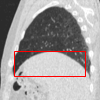}& \includegraphics[scale=.5]{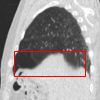}&\includegraphics[scale=.5]{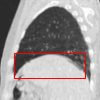}&\includegraphics[scale=.5]{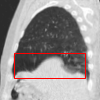}&\includegraphics[scale=.5]{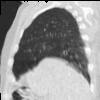}&\includegraphics[scale=.5]{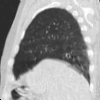}\\ \hline
\end{tabular}
\caption{Result of Example 4}
\label{table: first lung pics}
\end{table}
In addition to running this experiment on our algorithm, we again test decomposing the vector field matrix as in Example 1. We stacked the registration deformation vector fields from the reference image to all other frames in the video into one giant matrix over complex field. Then, running the complex low-rank and sparse component pursuit on the matrix gave the last two columns in Table \ref{table: first lung pics}. It is clear that the decomposed sparse matrix can barely capture any abnormal deformation as does the Beltrami descriptor. Our experiment once again shows that applying the algorithm to decompose on vector field matrices is not the best choice. \par

\begin{figure}
    \begin{subfigure}{.5\textwidth}
    \centering
    \includegraphics[scale = .4]{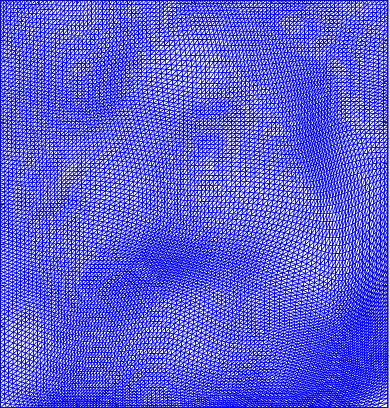}
    \caption{GT Low-Rank Deformation}
    \label{subfig: grid gt low rank lung}
    \end{subfigure}
    \begin{subfigure}{.5\textwidth}
    \centering
    \includegraphics[scale = .4]{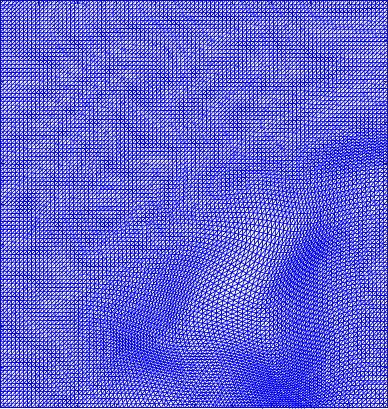}
    \caption{GT Sparse Deformation}
    \label{subfig: grid gt sparse lung}
    \end{subfigure}
    
    \begin{subfigure}{.5\textwidth}
    \centering
    \includegraphics[scale = .4]{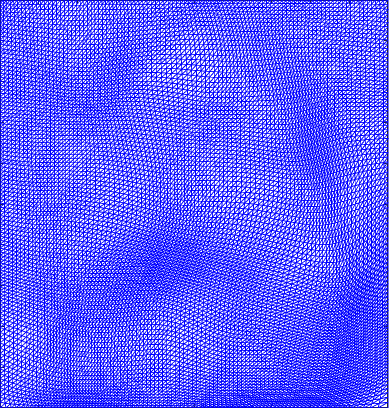}
    \caption{Recovered Low-Rank Deformation}
    \label{subfig: grid recovered low rank lung}
    \end{subfigure}
    \begin{subfigure}{.5\textwidth}
    \centering
    \includegraphics[scale = .4]{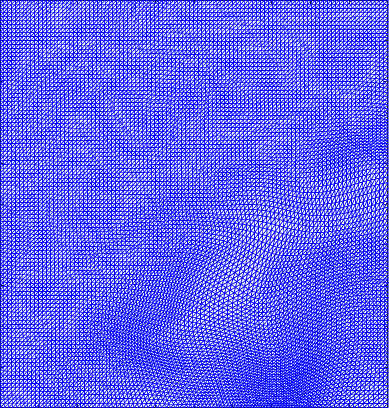}
    \caption{Recovered Sparse Deformation}
    \label{subfig: grid recovered sparse lung}
    \end{subfigure}
    \caption{Comparison Between Ground-Truth Maps and Recovered Maps}
    \label{fig: lung map comparison}
\end{figure}
Figure \ref{fig: lung map comparison} displays the 4 mappings as in Example 3 and 4. Again, we can see that after the decomposition of the Beltrami descriptor, the decpmposed mappings to large extent resemble the corresponding ones.

\begin{table}[t]
\centering
\begin{tabular}{|l|l|l|l|}
\hline
\begin{tabular}[c]{@{}l@{}} Frame \\ without abnormal \\ deformation\end{tabular} & \begin{tabular}[c]{@{}l@{}}Frame \\ with abnormal \\ deformation\end{tabular}& 
\begin{tabular}[c]{@{}l@{}}Recovered \\ Low-Rank \\ Frame \end{tabular}&\begin{tabular}[c]{@{}l@{}}Recovered \\ Sparse\\ Frame \end{tabular} \\ \hline
\includegraphics[scale=.75]{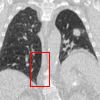}& \includegraphics[scale=.75]{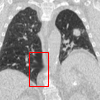}&\includegraphics[scale=.75]{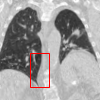}&\includegraphics[scale=.75]{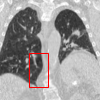}\\ \hline
\includegraphics[scale=.75]{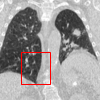}& \includegraphics[scale=.75]{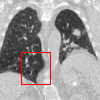}&\includegraphics[scale=.75]{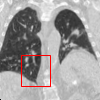}&\includegraphics[scale=.75]{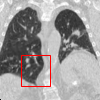}\\ \hline
\includegraphics[scale=.75]{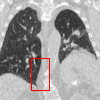}& \includegraphics[scale=.75]{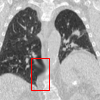}&\includegraphics[scale=.75]{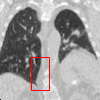}&\includegraphics[scale=.75]{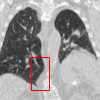}\\ \hline
\includegraphics[scale=.75]{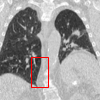}& \includegraphics[scale=.75]{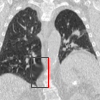}&\includegraphics[scale=.75]{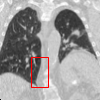}&\includegraphics[scale=.75]{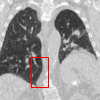}\\ \hline
\end{tabular}
\caption{Result of Example 5}
\label{table: second lung pics}
\end{table}

\bigskip

\noindent {\bf Example 5:} In this example, we test our proposed method on another medical video of a breathing lung with abnormal perturbation. The orginal video captures 36 cycles. The rank of the Beltrami descriptor matrix is 26. Our proposed method recovers the low rank matrix with rank 12. Table \ref{table: second lung pics} displays the results of one of the perturbation using our algorithm. Again, our proposed method effectively decompose the longitudinal deformation into the normal periodic component and abnormal component.\par

\bigskip

Finally, we summarize the rank of the decomposed sparse component for each example in Table 6. The ranks of the original input Beltrami descriptors are also recorded. Note that our proposed algorithm can effectively obtain the sparse component that reduces the rank. The rank of the sparse component closely resemble to the rank of the Beltrami descriptor of the video without abnormal perturbations.
\begin{table}
\centering
\begin{tabular}{|l|l|l|l|l|l|}
\hline
 &
  \begin{tabular}[c]{@{}l@{}}size of input\\ matrix\end{tabular} &
  \begin{tabular}[c]{@{}l@{}}rank of input\\ matrix without\\ perturbation\end{tabular} &
  \begin{tabular}[c]{@{}l@{}}rank of input\\ matrix with\\ perturbation\end{tabular} &
  \begin{tabular}[c]{@{}l@{}}rank of \\ recovered\\ low-rank matrix\end{tabular} &
  \begin{tabular}[c]{@{}l@{}} Decomposition \\ Time (s) \end{tabular} 
  \\ \hline
\begin{tabular}[c]{@{}l@{}}Synthetic Circle\\ Example\end{tabular} & $19602 \times 431$ & 24 & 47 & 27 & 167.06\\ \hline
\begin{tabular}[c]{@{}l@{}}First Heart\\ Example\end{tabular}      & $19602 \times 341$ & 11 & 15 & 11 & 44.80 \\ \hline
\begin{tabular}[c]{@{}l@{}}Second Heart\\ Example\end{tabular}     & $19602 \times 378$ & 18 & 36 & 20 & 98.41\\ \hline
\begin{tabular}[c]{@{}l@{}}First Lung\\ Example\end{tabular}       & $19602 \times 310$ & 10 & 23 & 10 & 21.98\\ \hline
\begin{tabular}[c]{@{}l@{}}Second Lung\\ Example\end{tabular}      & $19602 \times 360$ & 10 & 26 & 12 & 79.96\\ \hline
\end{tabular}
\caption{Summary of Recovered Rank by Our Proposed Algorithm}
\label{table: rank summary}
\end{table}


\section{Conclusion and future works}\label{sec: conclusion}
We address the problem of decomposing a longitudinal deformation into the normal periodic component and the abnormal irregular component. Our strategy is to represent the longitudinal deformation by the proposed Beltrami descriptor and apply RPCA on it. The low rank part effectively extracts the normal component, while the sparse part effectively captures the irregular deformation. The Beltrami descriptor describes the geometric information about the deformation, and hence performing the decomposition on the Beltrami descriptor yields meaningful results. In particular, we can prove that the extracted abnormal motion is guaranteed to be bijective under suitable choice of parameters. Extensive experiments on both synthetic and real data give encouraging results.

In this paper, we applied Low-Rank Sparse Matrix Pursuit to videos that are subject to periodic motion. We conjectured that this method, with some variations, can also be applied to videos subject to some constant deformations. For example, we might want to segment out a walking pedestrian with vehicles moving from left to right in the background. We hypothesized that the Beltrami descriptor could also be obtained using some registration techniques, such as the optical flow method. The descriptor of the moving background again should be of low-rank, and the pedestrian should be described by the sparse descriptor. This is our potential continuation of this project.


%
%

\bibliographystyle{spmpsci}      
\bibliography{citation.bib}   


\end{document}